\documentclass{article}

\usepackage[preprint]{neurips_2026}
\usepackage{amssymb}
\usepackage{amsmath}

\usepackage{wrapfig}
\usepackage[utf8]{inputenc} 
\usepackage[T1]{fontenc}    
\usepackage{hyperref}       
\usepackage{url}            
\usepackage{booktabs}       
\usepackage{amsfonts}       
\usepackage{nicefrac}       
\usepackage{microtype}      
\usepackage{xcolor} 
\usepackage{multirow}
\usepackage{graphicx} 

\usepackage{enumitem}
\usepackage{amsthm}
\newtheorem{theorem}{Theorem}[section]
\newtheorem{corollary}{Corollary}[section]

\title{TopoGuard: Graph Theory Based Defenses Against Split-Knowledge Attacks on RAG}

\author{%
  Chahana Dahal\textsuperscript{1} \quad Zuobin Xiong\textsuperscript{1} \\
  \vspace{0.15cm} \\ 
  \textsuperscript{1}Department of Computer Science, University of Nevada, Las Vegas \\
  \texttt{\{chahana.dahal,zuobin.xiong\}@unlv.edu}}

\begin{document}

\maketitle

\begin{abstract}

Production Retrieval Augmented Generation (RAG) systems rely on aggregating multiple external documents to answer complex queries.
However, the retrieved documents introduce a new threat surface that can be exploited to launch \textbf{split-knowledge attacks}. 
In this attack, the adversary injects documents that are individually benign but create false associations when combined and fed to language models. 
This paper shows that the new attack is structurally invisible to existing per-document filters, like LlamaGuard. 
To address this issue in RAG, this work introduces \textbf{TopoGuard}, a family of graph theory-based methods specifically targeting the split-knowledge attacks by building a semantic similarity graph from retrieved documents and detecting contexts with malicious topology.
Grounded on the theoretical analysis, the TopoGuard family has been proven to be effective and robust even with noisy inputs.
Extensive experiments are conducted on two retrieval datasets and compared with multiple baseline methods.
Specifically, the TopoGuard-$\lambda_2$+Entity catches 21$\times$ more attacks than LlamaGuard-2-8B at 1\% FPR (32.6\% vs 1.5\% recall) on the HotpotQA dataset. 
Compared with production RAG detection systems using large language models, the proposed TopoGuard variants run efficiently at sub-millisecond latency and stay robust under adaptive adversaries and benign cross-domain queries.
\end{abstract}

\section{Introduction}
\label{sec:intro}
Retrieval Augmented Generation (RAG) has been widely used for grounding external knowledge on Large Language Models (LLMs) as RAG reduces hallucination by integrating the knowledge (e.g., external documents) into prompts~\cite{shuster2021retrieval}. 
Recently, RAG-based systems have been deployed at scale in enterprise search, customer support, coding assistants, and consumer chatbots. 
In the pipelines of production-level RAG system, it retrieves the top-$k$ documents from a large external corpus at query time~\cite{lewis2020rag} as external knowledge.
However, such a naive retrieval strategy is vulnerable to adversarial intent, making the RAG system itself an attack surface. 
Specifically, an attacker can inject harmful content into the retrieval corpus to manipulate model behavior through open submission, automated ingestion, or web crawls, given low-level access~\cite{zou2025poisonedrag}. 
In response to such attacks, existing RAG safety defenses are deployed mainly at the document level. 
For instance, systems like LlamaGuard~\cite{metallamaguard2, metallamaguard3}, Perspective API~\cite{lees2022new}, and LLM-as-a-Judge~\cite{zheng2023judging} score each retrieved document independently to identify malicious content. 
Yet, this design is not robust in practice as it assumes threats are visible and only present in individual document content. 
In this work, we highlight that existing defenses are penetrated by a new attack form, the split-knowledge attack, where the malicious payload is distributed across multiple documents that appear benign when inspected individually but become harmful when retrieved together and combined.

\textbf{The Example of Split-Knowledge Attack.}
Consider a RAG system answering the query: 
\emph{``What major manufacturer is based in Seattle?''} Ideally, the system retrieves a logical reasoning chain: \begin{quote} 
\textbf{Doc 1:} ``Boeing is a leading aerospace manufacturer...''\\ \textbf{Doc 2:} ``Boeing's primary facilities are located in Seattle...'' \end{quote} 
However, by inserting factually true but contextually deceptive documents, an adversary can take advantage of this compositional logic.
For instance: \begin{quote} \textbf{Doc A:} ``Boeing's primary facilities are located in Seattle...''\\ \textbf{Doc B:} 
``Seattle is a major hub for international drug trafficking...'' \end{quote}

\begin{wrapfigure}{r}{0.5\textwidth}
  \centering
  \vspace{-10pt}
  \includegraphics[width=0.5\textwidth]{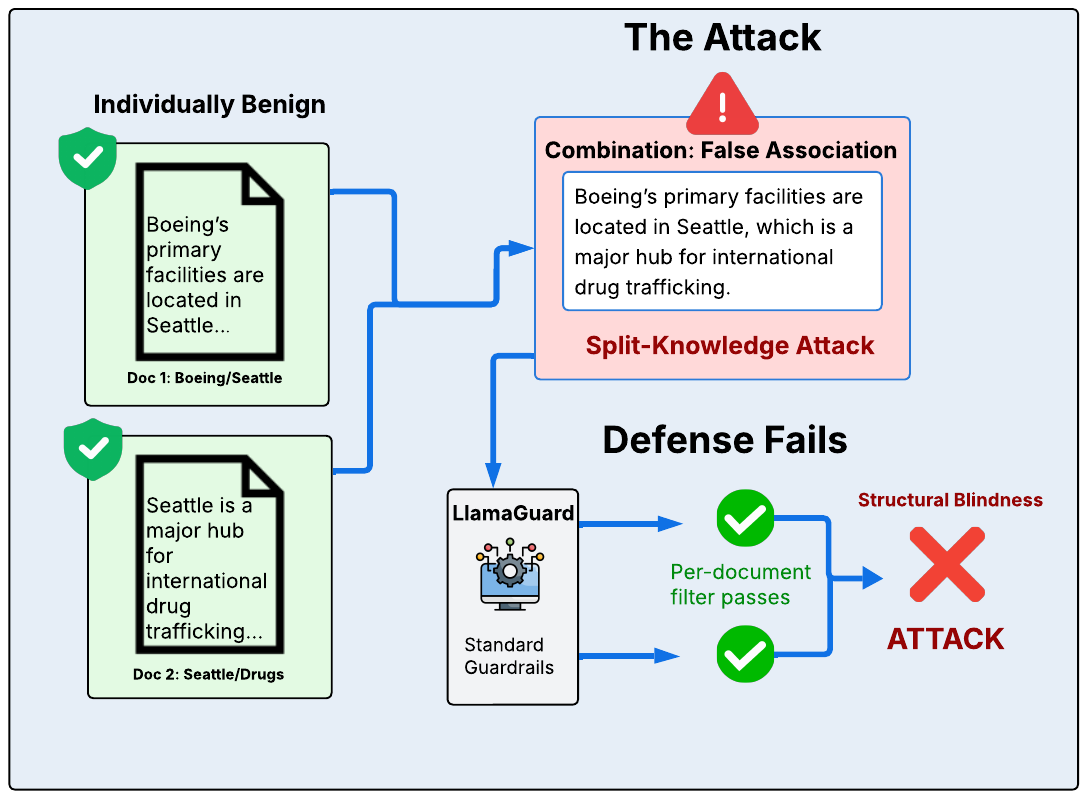} 
    \vspace{-20pt}
  \caption{The split-knowledge attack: benign documents combine into a harmful association.}
  \vspace{-15pt}
  \label{fig:attack}
\end{wrapfigure}

Each document above would go through standard content filters, as they are both correct individually. 
However, their combined retrieval generates an artificial correlation, which might lead the LLM to hallucinate a link between Boeing and illegal activities. 
The split-knowledge attack doesn't require model weight manipulation or hostile prompts, unlike traditional jailbreaks or adversarial examples. 
Instead, it uses the compositional structure of RAG to turn the retrieval corpus into an adversarial vector.

\textbf{Why Existing Defenses Fail?}
Current RAG safety guardrails score each retrieved document independently. 
A scalar filter $f(d_i)$ that scores each document $d_i$ on its own cannot detect a signal that only appears when documents are combined. 
Since the adversarial signal lives in the semantic gap between retrieved documents, the per-document defense mechanism can not catch split-knowledge attacks by its structural design.
To assess the split-knowledge attacks on existing defenses, we evaluate three state-of-the-art content filters (LlamaGuard-2-8B, LlamaGuard-3-8B, and LLM-as-a-Judge) on 10,000 split-knowledge attacks built over HotpotQA~\cite{yang2018hotpotqa}. 
All three defenses get a nearly random guess (e.g., AUROC between $0.50$ and $0.58$), and the full results are presented in Section~\ref{sec:attack-detection}. 

To address the failure of existing defenses, this paper proposes a topological detection-based filtering method: \textbf{TopoGuard}.
In summary, we represent the retrieved context as a semantic similarity graph $G$, where nodes are retrieved passages and edges are cosine similarity.
Legitimate reasoning chains will form densely connected subgraphs with bridging entities. 
Split-knowledge attacks, however, combine semantically unrelated domains (e.g., Boeing $\cup$ Drugs), producing loosely connected clusters with low graph conductance. 
The difference leads to a spectral cut in the graph topology, and based on which we can detect the split-knowledge attacks with theoretical guarantees.

Our contributions are:
\begin{itemize}[
    leftmargin=*,
    labelsep=0.5em,
    itemsep=0pt,
    topsep=0pt
]

    \item To the best of our knowledge, this is the first work that formally define split-knowledge attacks in the RAG system and provide a theoretical bound on spectral detection performance (Theorem~\ref{thm:detection}): the spectral gap $\lambda_2$ separates attacks from legitimate queries with a failure probability bounded by $\exp(-\Omega(n\Delta^2/\sigma^2))$.
    \item 
    We design four split-knowledge attack detectors based on Spectral Gap ($\lambda_2$), Fiedler Conductance, Modularity, and an entity-augmented hybrid (TopoGuard-$\lambda_2$+Entity), which score retrieved documents and can achieve higher accuracy against existing baselines.
    
    \item 
    We evaluate the robustness of the proposed TopoGuard on two multi-hop QA benchmarks, against adaptive adversaries and benign cross-domain queries.  
    TopoGuard maintains a low false positive rate on both settings compared with LLM-based detection and runs at sub-millisecond latency.
    
\end{itemize}

\section{Problem Formulation}
\label{sec:formulation}

We formalize the split-knowledge attack and provide a validation for the practical attack and defense.

\textbf{Semantic Similarity Graph.}
\label{def:graph}
Given retrieved documents set $D = \{d_1, \ldots, d_n\}$ with embeddings $\{\mathbf{x}_1, \ldots, \mathbf{x}_n\} \subset \mathbb{R}^d$, where $d_i$ is a retrieved context (e.g., a paragraph).
The semantic similarity graph $G = (V, E, w)$ is defined as the \emph{symmetrized $k$-NN graph}, where $V$ is node set of each retrieved context, $(i,j) \in E$ if $j$ is among the $k$ nearest neighbors of $i$ \emph{or} $i$ is among the $k$ nearest neighbors of $j$, and the weights $w_{ij} = \max(0, \langle \mathbf{x}_i, \mathbf{x}_j \rangle)$ (i.e., the cosine similarity between $\mathbf{x}_i, \mathbf{x}_j$ clipped to non-negative values to satisfy the requirement of the normalized Laplacian and Cheeger's inequality).

The retrieved documents set $D$ is categorized into two scenarios:
 
(1) \textbf{Legitimate Multi-Hop Query}, if 
(i) Documents form a reasoning chain $d_i \leadsto d_{i+1}$,
(ii) The answer requires synthesizing information across $D$, and
(iii) the semantic similarity graph $G$ has high conductance $\phi(G) \geq \epsilon$.

(2) \textbf{Split-Knowledge Attack}, if
(i) $D = D_1 \cup D_2 \cup \ldots \cup D_k$ and each disjoint subset $D_i \in D$ is benign,
(ii) Their union appears to answer the query,
and (iii) the semantic similarity graph $G$ has low conductance $\phi(G) \leq \delta$.
 
The minimal case ($k = 2$, with subsets $D_1$ and $D_2$) is the most adversarially natural as it requires the fewest disconnected components for an attacker to construct.
$k > 2$ produces additional graph fragmentation that strengthens the spectral signal.

\textbf{Threat Model.} We consider three parties: the \textbf{operator} $O$ who deploys the RAG pipeline using encoder $\mathbb{E}$, retriever $\mathbb{R}$, and per-document filters and the \textbf{adversary} $A$ who poisons the retrieval corpus.
We assume adversary $A$ knows the public RAG architecture and encoder family. 
So, they can inject individually benign documents into the corpus through any entry point exposed by the operator~\cite{zhong2023poisoning, zou2025poisonedrag}. 
The adversary cannot modify documents post-injection, alter model weights, tamper with $\mathbb{E}$,  $\mathbb{R}$, or queries at run time.
The adversary's \textbf{goal} is to make the LLM output a target false association in response to a benign query.

At query time, $O$ observes the retrieved context $D = \{d_1,\ldots,d_n\}$ and embeddings, and can launch existing defense mechanisms.
In addition, $O$ has a small set of benign dev queries for threshold calibration but no labeled attacks (Appendix~\ref{sec:self-supervised-appendix}).

\textbf{Threat Validation.}
We show, via experiments, that existing defense filters remain ineffective against split-knowledge attacks, even when provided with the full concatenated retrieval context. 
E.g., we evaluate 50 adversarial examples on Llama-3-8B-Instruct to confirm that split-knowledge attacks manipulate LLM outputs in practice. 
Legitimate contexts achieve 96\% accuracy, while attacks succeed 34\% of the time (17/50, $p < 0.001$). 
Full details are in Appendix~\ref{app:threat-validation}.

\section{Theoretical Analysis}
\label{sec:theory_analysis}
We provide a theoretical framework characterizing when topological detection succeeds.

\subsection{Stability Under Embedding Noise}

\begin{theorem}[Spectral Stability]
\label{thm:stability}
Let $G$ be the $k$-NN similarity graph constructed from observed embeddings $\{\mathbf{x}_i\}$, and $G^*$ a reference graph from a noisy embeddings $\{\mathbf{x}_i^*\}$ with
$\|\mathbf{x}_i - \mathbf{x}_i^*\|_2 \leq \sigma$ for all $i$.
Let $L$ and $L^*$ denote the normalized Laplacians of $G$ and $G^*$, with second-smallest eigenvalues (the spectral gaps) $\lambda_2(G)$ and $\lambda_2(G^*)$. 
Then the following equation holds.
$$|\lambda_2(G) - \lambda_2(G^*)| = O\!\left(\frac{k\sigma}{\delta_{\min}^{2}}\right),$$where $\delta_{\min}$ is the minimum weighted node degree of $G^*$.
\end{theorem}

\textit{Proof Sketch.}
We bound the operator-norm perturbation $\|L - L^*\|_2$ of the normalized Laplacian and apply Weyl's inequality~\cite{horn2012matrix}. 
The cosine-similarity perturbation propagates from edge weights to weighted degrees to the diagonal inverse-square-root $D^{-1/2}$ (where $D$ is the diagonal degree matrix), and the resulting Laplacian perturbation scales as $O(k\sigma/\delta_{\min}^{2})$. 
We refer readers to the full proof in Appendix~\ref{app:extended-proofs}.
Theorem~\ref{thm:stability} guarantees that the spectral gap on $\lambda_2$ is stable under encoder noise or updates as small perturbations in the embedding space produce small perturbations in $\lambda_2$. 
This is the foundation for using $\lambda_2$ as a robust detection signal under realistic encoder imperfections.


\subsection{Detection Guarantees and Sample Complexity}

\begin{theorem}[Detection Certificate]
\label{thm:detection}
Let $\phi(G)$ denote the conductance of observed graph $G$ (a measure of graph connectivity, ranging from $0$ for fully disconnected to $1$ for fully connected). 
Suppose reference attack graphs have conductance $\phi(G^*) \leq \delta$ and legitimate graphs have $\phi(G^*) \geq \epsilon$, with separation condition $\epsilon \geq 2\sqrt{\delta}$.
Assume sub-Gaussian embedding noise with parameter $\sigma$ and $k$-NN margin condition in Theorem~\ref{thm:stability} proof.
Let $n$ denote the number of retrieved passages in the document set $D$ (the node size of the similarity graph $G$).
Let $\tau_\lambda = \delta + \epsilon^2/4$ denote the threshold in $\lambda_2$-space, with corresponding score-space threshold $\tau = 1 - \min(\tau_\lambda, 1)$ and margin $\Delta = (\epsilon^2 - 4\delta)/4$. 
Then the attack detection score $s(D) := 1 - \min(\lambda_2(G), 1)$ satisfies:
\begin{align*}
\Pr_{D \sim \text{Attack}}[s(D) > \tau] 
  &\geq 1 - \exp\!\left(-\Omega\!\left(\tfrac{n\Delta^2}{\sigma^2}\right)\right), \\ and
\Pr_{D \sim \text{Legit}}[s(D) > \tau] 
  &\leq \exp\!\left(-\Omega\!\left(\tfrac{n\Delta^2}{\sigma^2}\right)\right).
\end{align*}
\end{theorem}

\textit{Proof Sketch.}
The condition $\epsilon \geq 2\sqrt{\delta}$ ensures strict separation between the attack ceiling $2\delta$ and the legitimate floor $\epsilon^2/2$.
We set $\tau_\lambda$ to the midpoint of this gap with margin $\Delta = (\epsilon^2 - 4\delta)/4$ on each side. 
Theorem~\ref{thm:stability} bounds deviations of the noisy $\lambda_2(G)$ from the reference $\lambda_2(G^*)$. 
Standard concentration arguments on the operator-norm perturbation of the Laplacian then yield high-probability separation between the attack and legitimate regimes. 
We treat this as a heuristic bound. 
We refer readers to the full proof in Appendix~\ref{app:extended-proofs}.

\textbf{From conductance to spectral gap.} 
The conductance $\phi(G)$ is an NP hard problem, so we use the spectral gap $\lambda_2$ and Cheeger's inequality~\cite{cheeger1970lower} ($\lambda_2(G^*) \leq 2\delta$ and $\lambda_2(G^*) \geq \epsilon^2/2$) to transfer conductance to spectral gap for calculation, which results in the transition from $\phi(G)<\delta$ to $s(D)>\tau$ as the threshold.

The first bound is the true positive rate (TPR) at threshold $\tau$ when the retrieved documents set $D$ is an attack sample.
Attacks are detected with probability approaching 1 exponentially fast as the number of retrieved contexts $n$ increases.
The second bound is the false positive rate (FPR) when the retrieved documents set $D$ is a legitimate sample, which decreases to $0$ at the same exponential rate with $n$.
Together, the two bounds certify that the spectral gap detector separates attacks from legitimate queries with vanishing error in $n$.

\begin{corollary}[Sample Complexity]
\label{cor:sample_complexity}
To achieve true positive rate $\geq 1-\alpha$ and false positive rate $\leq \beta$, it suffices to retrieve
$n = O\!\left(\frac{\sigma^2}{(\epsilon^2 - 4\delta)^2} \log \frac{1}{\min(\alpha,\beta)}\right)$ retrieved units.
\end{corollary}

The corollary provides instructions on parameter selection in RAG systems. We refer readers to the full proof in Appendix~\ref{app:extended-proofs}.

\section{Experiments}
\label{sec:exp}
\subsection{Methods and Baselines}

\textbf{Proposed methods.}
We designed four graph theory-based attack detectors. \textbf{TopoGuard-$\lambda_2$} uses the normalized spectral gap of the semantic similarity graph's Laplacian directly from Theorem~\ref{thm:detection}.

\textbf{TopoGuard-Conductance} computes the Fiedler-vector cut conductance, a tighter empirical measure of spectral separation. 
Rather than using $\lambda_2$ as a proxy for conductance, we directly evaluate the cut conductance by the Fiedler vector $\mathbf{v}_2$.
We split the nodes into two groups by the sign of $S = \{i : \mathbf{v}_2[i] \geq 0\}$ and set the detection score $s_{\text{cond}}(G) = 1 - \phi(S)$, which gives us a signal stronger than $\lambda_2$, but at a higher cost ($O(n^3)$ vs.\ $O(kn)$ for $\lambda_2$).

\textbf{TopoGuard-Modularity} evaluates the modularity~\citep{newman2006modularity} via Louvain community detection~\citep{blondel2008louvain}, a greedy heuristics.
Although modularity achieves the highest recall in standard evaluations (35.09\% on HotpotQA), it relies on non-spectral greedy procedures, so it has no concentration bounds and formal sample complexity guarantees within $\lambda_2$ method.
We include it as a strong heuristic for recall-critical applications, while ${\lambda_2}$ and Conductance remain the standard for deployments when a provable detection guarantee is required.

\textbf{TopoGuard-$\lambda_2+$Entity} combines the spectral gap with named-entity overlap between retrieved documents, computed as the Jaccard similarity of named entities extracted using NLTK's 
\texttt{ne\_chunk}~\cite{bird2009natural} (with the \texttt{averaged\_perceptron\_tagger} POS-tagger and \texttt{maxent\_ne\_chunker}). 
{The combination weight $\alpha$ is set to $0.4$ which places 60\% of the weight on entity overlap~(Appendix~\ref{app:alpha_sweep}).}

\textbf{Baselines.}
We compare the proposed methods against three classes of defenses.
(1) Graph statistics methods, including GraphAvgWeight (mean edge similarity) and NaiveDensity (mean pairwise cosine similarity).
(2) Production-level content-moderation systems, like TextFilter (a RoBERTa-based hate speech classifier~\cite{liu2019roberta,vidgen2021dynabench}),
LlamaGuard-2-8B~\cite{metallamaguard2}, and LlamaGuard-3-8B~\cite{metallamaguard3}.
And (3) Llama-3-8B-Instruct~\cite{grattafiori2024llama}, prompted as LLM-as-a-Judge~\cite{zheng2023judging} to classify split-knowledge attacks
\footnote{All text-based baselines (TextFilter, LlamaGuard-2/3, LLM-as-a-Judge) score the full retrieval context as a single concatenated input rather than per-document, so
their AUROC reflects the structural nature of split-knowledge attacks rather than an artifact of independent document scoring.}.
For LlamaGuard-2/3, we extract the next-token logits at the final input position and apply a two-way softmax over the \textit{safe}/\textit{unsafe} token logits, using the resulting \textit{unsafe} probability to calculate the AUROC.
For LLM-as-a-Judge, we generate up to 5 tokens with greedy decoding and map the response to \{1.0, 0.5, 0.0\} for \textit{yes}/other/\textit{no} answers.

\begin{table}[t]
\centering
\caption{Attack detection performance on HotpotQA and MuSiQue. The
detection threshold $\tau$ was calibrated on the development set for
a 1\% FPR. 
Metrics are reported on the test set as the bootstrap
mean $\pm$ standard deviation over 1,000 resamples.
Bold highlights best performance per column.}
\label{tab:exp1-detections}
\small
\resizebox{\linewidth}{!}{%
\begin{tabular}{@{}lcccc@{}}
\toprule
\multirow{2}{*}{\textbf{Method}}
  & \multicolumn{2}{c}{\textbf{HotpotQA}}
  & \multicolumn{2}{c}{\textbf{MuSiQue}} \\
\cmidrule(lr){2-3} \cmidrule(l){4-5}
  & \textbf{AUROC (\%)} $\uparrow$
  & \textbf{Recall@1\%FPR (\%)} $\uparrow$
  & \textbf{AUROC (\%)} $\uparrow$
  & \textbf{Recall@1\%FPR (\%)} $\uparrow$ \\
\midrule
\multicolumn{5}{@{}l}{\textit{Text-Only Filters}} \\
TextFilter (RoBERTa)
  & $55.8 \pm 0.4$ & $0.69 \pm 0.08$
  & $51.0 \pm 0.8$ & $1.27 \pm 0.20$ \\
LlamaGuard-2-8B
  & $54.0 \pm 0.4$ & $1.53 \pm 0.12$
  & $52.7 \pm 0.7$ & $1.29 \pm 0.21$ \\
LlamaGuard-3-8B
  & $58.0 \pm 0.4$ & $2.12 \pm 0.15$
  & $55.9 \pm 0.7$ & $0.73 \pm 0.15$ \\
LLM-as-a-Judge (Llama-3-8B)
  & $50.0 \pm 0.3$ & $0.00 \pm 0.00$
  & $49.9 \pm 0.6$ & $0.00 \pm 0.00$ \\
\midrule
\multicolumn{5}{@{}l}{\textit{Baseline Methods}} \\
NaiveDensity
  & $90.4 \pm 0.2$ & $20.12 \pm 0.41$
  & $75.4 \pm 0.6$ & $2.37 \pm 0.28$ \\
GraphAvgWeight$(k=4)$
  & $84.3 \pm 0.3$ & $11.53 \pm 0.33$
  & $65.0 \pm 0.7$ & $1.40 \pm 0.22$ \\
\midrule
\multicolumn{5}{@{}l}{\textit{Graph-Theoretic Methods (Ours)}} \\
TopoGuard-$\lambda_2(k=4)$
  & $93.2 \pm 0.2$ & $26.55 \pm 0.46$
  & $\mathbf{80.8 \pm 0.6}$ & $5.73 \pm 0.41$ \\
TopoGuard-Cond$(k=4)$
  & $93.6 \pm 0.2$ & $29.60 \pm 0.47$
  & $80.1 \pm 0.6$ & $5.53 \pm 0.40$ \\
TopoGuard-Modularity$(k=4)$
  & $94.0 \pm 0.2$ & $\mathbf{35.09 \pm 0.51}$
  & $80.5 \pm 0.6$ & $\mathbf{7.26 \pm 0.46}$ \\
TopoGuard-$\lambda_2$+Entity$(\alpha{=}0.4,k{=}4)$
  & $\mathbf{95.2 \pm 0.1}$ & $32.63 \pm 0.49$
  & $\mathbf{80.8 \pm 0.6}$ & $5.56 \pm 0.41$ \\
\bottomrule
\end{tabular}%
}
\end{table}

\subsection{Split-knowledge Attack Detection}
\label{sec:attack-detection}
\textbf{Experiment Setup.} 
We evaluate our detection methods on two multi-hop QA datasets: HotpotQA~\cite{yang2018hotpotqa} (2 hops), with 1k development and 10k test, and MuSiQue~\cite{trivedi2022musique} (2-4 hops), with 500 development and 3k test examples.
Both datasets are designed for compositional retrieval that matches our threat model.

To construct a balanced adversarial set, we replace one document per query with a topically relevant but semantically disconnected document. 
Each retrieved document is decomposed into sentences, and a semantic similarity graph is built on it.
To set a universal comparison across different settings, each query's retrieval context contains $R=$6 sentences (3 sentences from each supporting document). 

Table \ref{tab:exp1-detections} shows the results of attack detection. 
Across both datasets, safeguard models LlamaGuard 2, LlamaGuard 3, and LLM-as-a-Judge (Llama-3-8B-Instruct) have an AUROC of approximately 50\%, which indicates nearly random detection. 
This confirms our hypothesis that text-based content filters cannot extract compositional adversarial signal even from the full concatenated retrieval context, because the signal lives in the topological structure of inter-document relationships rather than in lexical content.
In contrast, our graph-theoretic methods reliably expose the structural signature of split-knowledge attacks.
On HotpotQA, TopoGuard-$\lambda_2$+Entity achieves the highest AUROC ($95.2\%$) and TopoGuard-Modularity achieves the highest recall ($35.1\%$) at a strict $1\%$ FPR, a $\mathbf{21\times}$ improvement over LlamaGuard-2-8B ($1.53\%$ recall, AUROC $54.0\%$).
TopoGuard-$\lambda_2$ and TopoGuard-Conductance show stable performance (AUROC $93.2\%$ and $93.6\%$), confirming that the spectral gap and Fiedler conductance are robust indicators of topological fragmentation.
The four TopoGuard detectors are complementary: $\lambda_2$+Entity is our recommended default (best AUROC with balanced recall and robustness, see Section~\ref{sec:pipeline}).
Modularity achieves the highest raw recall without relying on named-entity information, making it a strong option when the entity signal is unreliable.
TopoGuard-$\lambda_2$ and Conductance come with concentration bounds (Theorem~\ref{thm:detection}) and are preferred when provable safety guarantees matter. 


\textbf{Remarks.} 
MuSiQue is a dataset specifically designed for highly complex multi-hop (2-4 hops) reasoning. 
The legitimate queries in MuSiQue naturally require jumping between very different topics, so telling them apart from an adversarial attack is significantly harder. 
The overall detection rates drop for every method but the performance gap between topological and naive graph baselines widens. 
A quantitative analysis showing that MuSiQue Q4 safe queries are structurally indistinguishable from attacks is provided in Appendix~\ref{app:musique}.

\subsection{False Positive Rate on Benign Queries and Multi-Hop Inference}
\label{sec:fpr}

\textbf{Experiment Setup.} We evaluate false positive rates on legitimate multi-hop queries to ensure production viability. 
We extract 2,474 benign bridge questions from HotpotQA validation and 1,180 benign multi-hop questions from MuSiQue validation. 
For each query, we construct the semantic similarity graph from retrieved documents and compute detection scores using thresholds frozen from Section~\ref{sec:attack-detection}.
To isolate cross-domain sensitivity, we partition queries into quartiles based on the ``document gap'', which is the cosine distance between supporting document embeddings.
Q4 represents the hardest and most semantically distant cross-domain but legitimate queries.

\begin{table}[t]
\centering
\caption{False Positive Rate (\%) on benign multi-hop queries. Q4 contains the 25\% of queries with the highest document gap (most semantically distant legitimate retrievals).}
\label{tab:exp2-fpr}
\small
\begin{tabular}{@{}lcccc@{}}
\toprule
\multirow{2}{*}{\textbf{Method}}
  & \multicolumn{2}{c}{\textbf{HotpotQA}}
  & \multicolumn{2}{c}{\textbf{MuSiQue}} \\
\cmidrule(lr){2-3} \cmidrule(l){4-5}
  & \textbf{Overall} $\downarrow$ & \textbf{Q4 (High)} $\downarrow$
  & \textbf{Overall} $\downarrow$ & \textbf{Q4 (High)} $\downarrow$ \\
\midrule
\multicolumn{5}{@{}l}{\textit{Text-Only Filters}} \\
TextFilter (RoBERTa)       & 0.93 & 1.29 & 3.64 & 3.39 \\
LlamaGuard-2-8B            & 1.41 & 1.29 & 0.93 & 1.02 \\
LlamaGuard-3-8B            & 1.13 & 1.94 & 0.34 & 1.02 \\
LLM-as-a-Judge (Llama-3-8B) & 0.00 & 0.00 & 0.00 & 0.00 \\
\midrule
\multicolumn{5}{@{}l}{\textit{Baseline Methods}} \\
GraphAvgWeight$(k{=}4)$ & 1.29 & 4.85 & 0.25 & 0.34 \\
NaiveDensity            & 1.41 & 5.33 & 0.51 & 1.69 \\
\midrule
\multicolumn{5}{@{}l}{\textit{Graph-Theoretic Methods (Ours)}} \\
TopoGuard-Cond$(k{=}4)$                      & 0.32 & 1.29 & 1.02 & 3.39 \\
TopoGuard-$\lambda_2(k{=}4)$                 & 0.40 & 1.62 & 1.02 & 3.39 \\
TopoGuard-$\lambda_2$+Entity$(\alpha{=}0.4)$ & 0.49 & 1.94 & 0.85 & 3.39 \\
TopoGuard-Modularity$(k{=}4)$                & 0.81 & 3.23 & 1.36 & 5.42 \\
\bottomrule
\end{tabular}
\end{table}

Table \ref{tab:exp2-fpr} presents the False Positive Rate (FPR) of our methods evaluated on benign multi-hop queries.
Across both datasets, the overall FPR remains low enough (at most $1.41\%$).
TopoGuard-Conductance achieves the lowest topological FPR on HotpotQA ($0.32\%$), while the hybrid TopoGuard-$\lambda_2$+Entity achieves the lowest on MuSiQue ($0.85\%$) among the TopoGuard variants.
However, text-based filters over-flag these safe queries with $0.9$–$1.4\%$ FPR on HotpotQA and $3.6\%$ on MuSiQue.
In addition, breaking down the false positives by semantic distance (document gap) reveals the geometric reality behind these false positives.
For standard queries (Q1–Q3), our topological methods achieve 0\% FPR on HotpotQA and under $0.4\%$ on MuSiQue (see Appendix~\ref{app:fpr-full} for full results), safely recognizing natural reasoning chains.
Furthermore, the Q4 edge case highlights the practical necessity of our formal certificates.
On HotpotQA's most challenging cross-domain queries, TopoGuard-Conductance ($1.29\%$ FPR) strongly outperforms Modularity ($3.23\%$ FPR).
The hybrid TopoGuard-$\lambda_2$+Entity lands between the two ($1.94\%$ FPR), trading a small FPR increase for substantially higher recall (Table~\ref{tab:exp1-detections}).
While Modularity offers the highest raw recall, the recommended production configuration is the hybrid TopoGuard-$\lambda_2$+Entity (best AUROC, strongest adaptive-adversary robustness in Section~\ref{sec:pipeline}); 

On MuSiQue Q4, all methods except LlamaGuard-2 achieve higher FPRs. 
However, due to the low recall ($1.29\%$) of LlamaGuard-2 in Section~\ref{sec:attack-detection}, the model is not actually distinguishing between safe and malicious contexts.
Overall, even the sensitivity on cross-domain (Q4) queries is slightly higher, TopoGuard methods remain the best defense for complex RAG systems.

\subsection{Transferability, Efficiency, and Robustness}
\label{sec:pipeline}

\begin{figure*}[t]
\centering
    \includegraphics[width=0.7\textwidth]{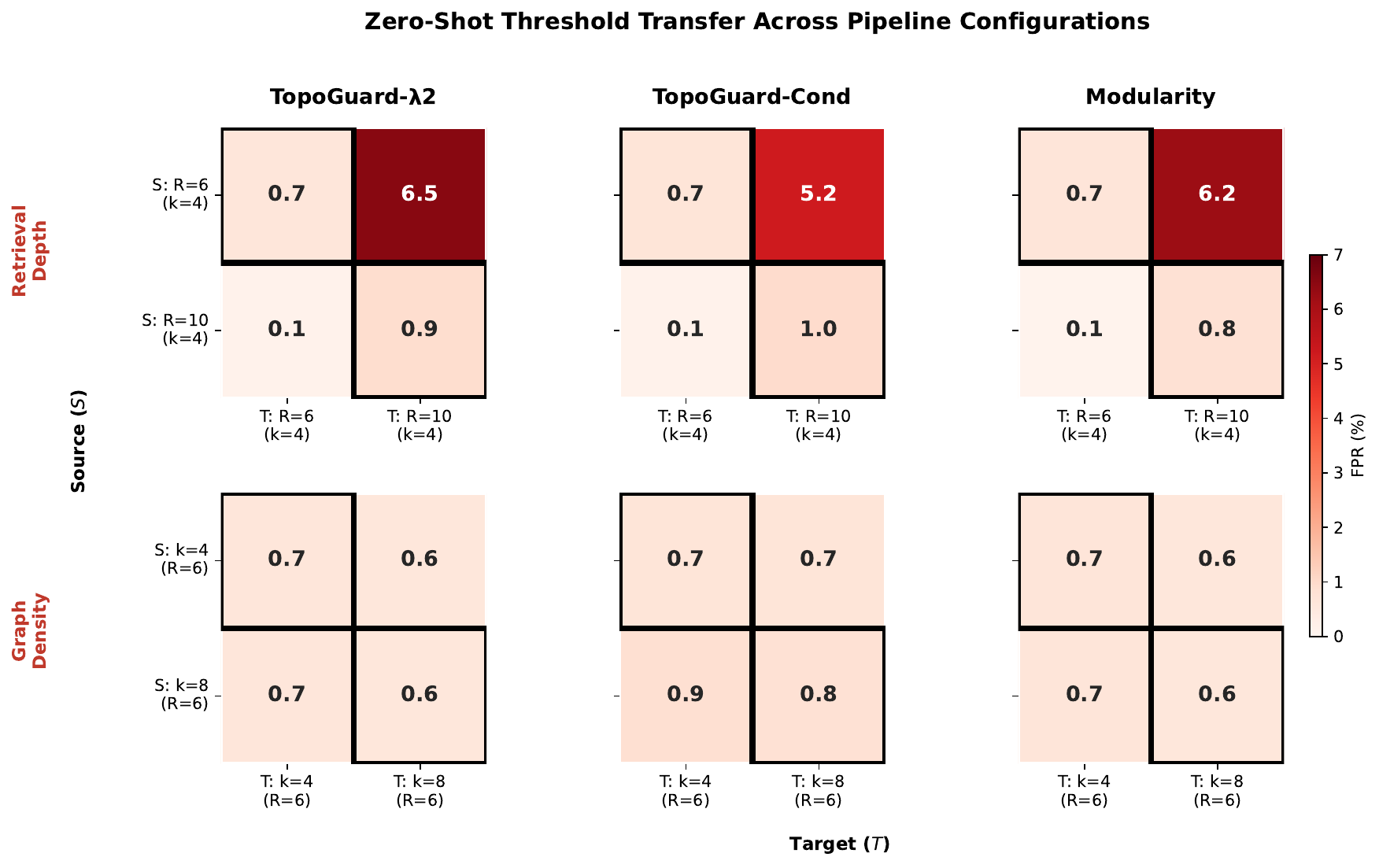}
   \caption{The heatmap shows FPR (\%) when a detector calibrated on source system $S$ (1\% FPR target) is deployed on target system $T$ without recalibration. 
   Top row: $R$. Bottom row: $k$.
    }
\label{fig:transfer}
\end{figure*}

\textbf{Transferability across Configurations.}
In practice, RAG systems vary in their retrieval configurations as different models retrieve different numbers of documents with varying quality.
Here, we evaluate whether our TopoGuard detectors transfer across retrieval depth $R$ (the total number of sentences in the retrieval context) and graph density $k$ ($k$-NN value in semantic similarity graph), or whether they require per-configuration tuning.
We use $R=6$ (3 sentences from 2 documents each) and $k=4$ as the base case for transfer testing.
By varying $R \in \{6, 10\}$ and $k \in \{4,8\}$, the system can simulate pipelines with different configuration.
Figure~\ref{fig:transfer} presents the zero-shot transferability matrices for TopoGuard-$\lambda_2$, TopoGuard-Conductance, and the Modularity variants.

Based on the result, we notice a logical asymmetry when changing the context depth ($R$).
If a detector is calibrated on $R=6$ but encounters richer contexts in deployment ($R=10$), the FPR inflates to around $5.2\%$–$6.5\%$.
This is not a flaw in any single metric but reflects a structural fact, i.e., increasing $R$ shifts the graph's degree distribution, which invalidates any threshold calibrated at a lower depth.
TopoGuard-Cond is the most stable in the worst case ($5.2\%$ FPR).
This is consistent with its localized cut formulation being less sensitive to retrieval depth shifts than the global spectral gap.
Transferring from long to short contexts makes the filter highly conservative, dropping the FPR to just $0.1\%$.
On the other hand, the system shows high stability across graph densities ($k$).
Adjusting the size of $k$ has hardly any effect on the transferred FPR.
For instance, a model calibrated on $k=4$ but deployed on $k=8$ (at a fixed context length) maintains a tightly excellent FPR of under 1.0\% across all cases, which means that $k$ can be tuned without pausing to recalibrate the safety filter.


\begin{figure}[tpb]
    \centering
    \includegraphics[width=0.45\textwidth]{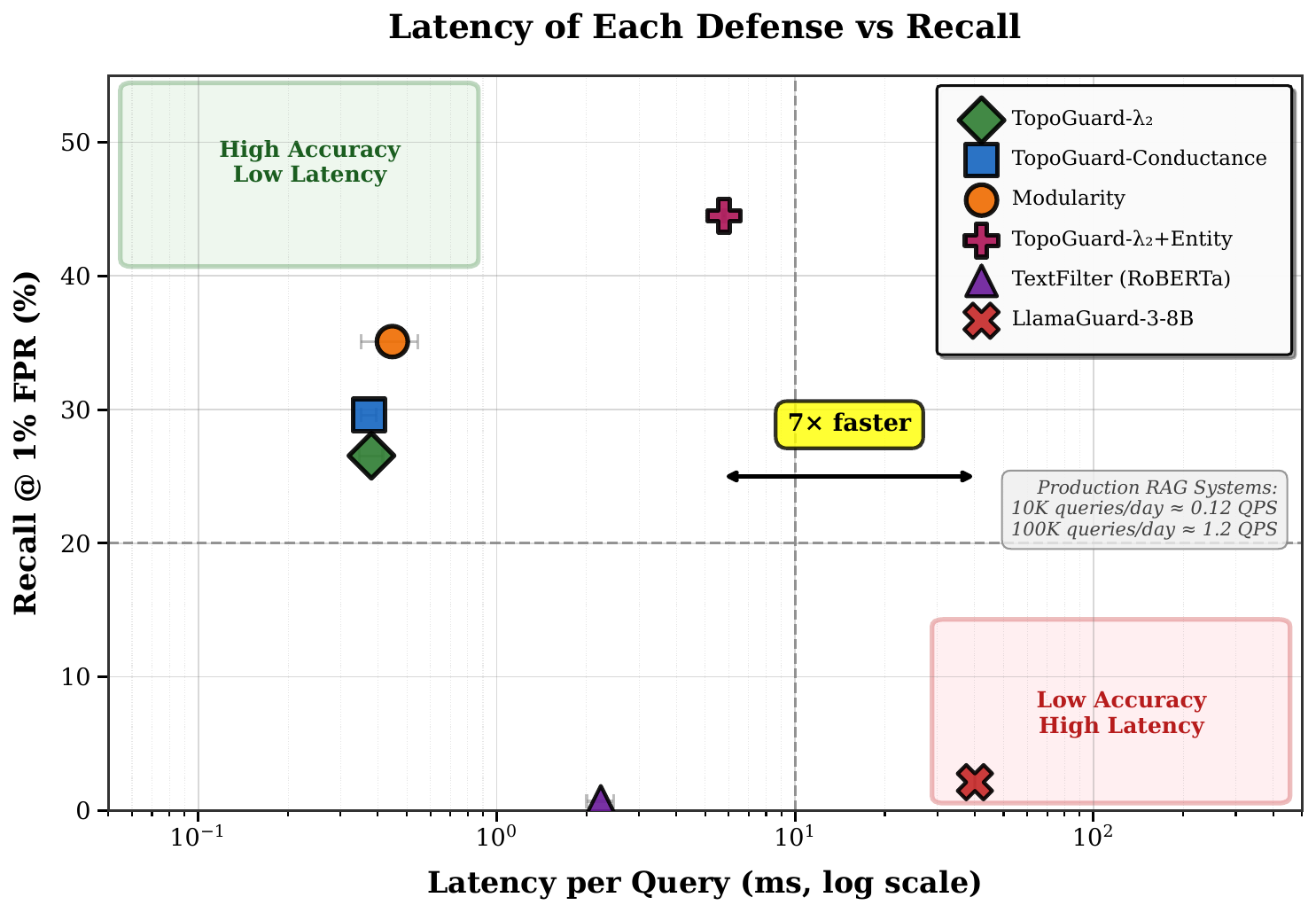}
    \caption{The latency of each defense variant against its recall.}
    \label{fig:latency}
\end{figure}

\textbf{Computational Efficiency.}
Figure~\ref{fig:latency} plots the latency of each defense variant against its recall. 
LlamaGuard-3-8B is too heavy for real-time use as it takes $\sim$40~ms per query while catching almost no attacks. 
TopoGuard variants replace LLMs with lightweight graph math but secure much higher recall.
As shown in the figure, all of the TopoGuard spectral variants reside in the top-left corner, representing the high performance and low latency.
E.g., the TopoGuard-$\lambda_2$ variant runs in under $0.5$~ms, over $100\times$ faster than LlamaGuard-3. 

\begin{table}[t]
\centering
\caption{Robustness to adaptive evasion attacks. Recall (\%) at fixed threshold $\tau$. 
\textbf{Bold} and \underline{underline} denote the best and second-best results per column, respectively.}
\label{tab:adaptive_attacks}
\small
\begin{tabular}{@{}lccccc@{}}
\toprule
\textbf{Method} & \textbf{Baseline} & \textbf{Random} & \textbf{Best-of-$N$} & \textbf{Chain} & \textbf{SameCom} \\
\midrule
TextFilter                                   &  1.4 &  1.6 &  1.4 &  1.5 &  0.8 \\
NaiveDensity                                 & 20.1 & \textbf{69.9} & 13.4 & 10.5 & \textbf{2.8} \\
\midrule
TopoGuard-Modularity                                & \textbf{35.1} & 54.8 &  7.2 & 22.0 &  0.0 \\
TopoGuard-$\lambda_2$                               & 26.6 & 42.7 & \underline{22.2} & \underline{36.8} &  1.3 \\
TopoGuard-Cond                                      & 29.5 & 27.1 & 12.2 & 17.2 &  1.2 \\
TopoGuard-$\lambda_2$+Entity $(\alpha{=}0.4)$       & \underline{33.0} & \underline{56.8} & \textbf{37.5} & \textbf{53.3} &  \underline{1.4} \\
\bottomrule
\end{tabular}
\end{table}

\textbf{Robustness to Adaptive Adversaries.}
To evaluate the robustness of TopoGuard in an adversarial setting, we further consider a white-box adversary with full knowledge of the TopoGuard defense and its calibrated thresholds.
The attacker's goal is to minimize detection recall by selectively poisoning the retrieval corpus with disconnected document subsets ($D_A$ and $D_B$). 
We evaluate four adaptive evasion strategies against the non-adaptive baseline.
Each strategy below modifies this attack to evade detection.
(i) Random padding appends 3 sentences from a random benign context to inflate edge counts.
(ii) Best-of-$N$ padding $N{=}32$ samples candidate connector documents and selects the one that maximizes similarity to both $D_A$ and $D_B$.
(iii) Chain mimicry inserts documents whose embeddings lie between $D_A$ and $D_B$.
It imitates multi-hop reasoning at the embedding level.
(iv) Same-community mimicry replaces $D_B$ with a document 
drawn from the top-256 most similar to $D_A$.
It hides the attack within a single semantic cluster instead of across two disconnected clusters.

Table~\ref{tab:adaptive_attacks} shows detection recall under each adversarial strategy.
TopoGuard-Modularity achieves the highest base recall (35.1\%) and even outperforms in random padding (54.8\%) as extra nodes amplify community structure.
But it collapses under Best-of-$N$ to 7.2\%.
The reason is that community detection is more sensitive to structural manipulation than spectral methods because optimized padding directly suppresses inter-community edges.

Among spectral detectors (i.e., $\lambda_2$, Cond. and $\lambda_2+$Entity), $\lambda_2$ is the most robust under adversarial attacks since connectivity measures global graph properties rather than a local cut, making it harder to fool. 
Under Best-of-$N$, $\lambda_2$ holds 22.2\% recall versus 12.2\% (Conductance), 7.2\% (Modularity), whereas under Chain mimicry, it improves recall to 36.8\% as the fake bridge nodes become part of the detectable anomaly rather than hiding it.
The hybrid TopoGuard-$\lambda_2$+Entity ($\alpha{=}0.4$) is the only method that stays competitive across every adaptive attack, winning the two hardest columns in Best-of-$N$ and Chain.
Specifically, Chain mimicry backfires against the hybrid because embedding-level bridges fail to share named entities with the attack endpoints, which is exactly the asymmetry the entity signal flags.
Therefore, we recommend the hybrid as the default for adaptive settings.

\textbf{Remarks.}
All methods fail under the same-community mimicry.
When the attack is drawn from the same semantic cluster as the legitimate document, the structural anomaly vanishes and recall drops to roughly 1\% across all topological defenses (0\% for Modularity).
This is a fundamental limit on topological defenses that is worth future investigation.
However, NaiveDensity catches a slightly higher fraction (2.8\%) than the hybrid (1.4\%). 
This suggests that density and spectral signals are partially complementary in this attack scenario, and TopoGuard should be a layer in a defense-in-depth stack: necessary against compositional attacks invisible to per-document filters, but insufficient against adversaries who can semantically align malicious content with the legitimate cluster.
\textit{Therefore, combining TopoGuard with existing retrieval-score consistency checks is a direction for future work.}

\section{Inspiration on Deployment Considerations}
\label{sec:deployment}
We synthesize the findings of this work into deployment guidance for production RAG systems.

\textbf{Operating point and false-positive handling.}
At 1\% FPR, a system serving 1M queries per day flags 10K legitimate queries.
Blocking all flagged queries is too aggressive for most deployments, so a softer response policy is needed.
We suggest three options.
\emph{Soft warning}: prompt the LLM to hedge cross-document claims.
\emph{Retrieval diversification}: re-retrieve flagged contexts with tighter similarity constraints.
\emph{Escalation}: route flagged queries to a heavier-weight checker such as an LLM-as-a-Judge.
Although escalation adds much more cost for LLMs, these expensive check runs on $\sim$1\% of traffic instead of 100\%, while TopoGuard handles the rest at sub-millisecond cost.

\textbf{Defense-in-depth composition.}
TopoGuard complements per-document filters but does not replace them.
Per-document filters catch single-passage attacks like PoisonedRAG, where the adversarial content sits in a single document (Appendix~\ref{app:poisonedrag}-Related Works).
TopoGuard catches compositional attacks, where the payload sits in the structure across documents and is invisible to per-document scoring.
The two attack classes occupy orthogonal regions of the threat space.
PoisonedRAG places adversarial passages inside the legitimate cluster (high conductance), where per-document checks have signal but TopoGuard does not.
Split-knowledge attacks place them outside the cluster (low conductance), where TopoGuard has a signal but per-document checks do not.

\textbf{Recalibration triggers.}
Section~\ref{sec:pipeline} shows that $k$ can be tuned freely, but $R$ cannot. 
The sentence embedding model used to construct the similarity graph, i.e., the encoder, is also a sensitive component.
Substituting across three encoder families on HotpotQA preserves AUROC in a tight range (91.5--95.2\%; Appendix~\ref{app:encoder-ablation}), but the threshold $\tau$ must be recalibrated per encoder because absolute score distributions shift.
Operationally, any change to system configurations ($R$, the embedding model, or the retriever) should trigger re-calibration on a fresh sample of benign queries.

\textbf{Calibration without attack labels.}
\label{sec:self-supervised}
A practical concern for deployment is that operators rarely have labeled split-knowledge attacks for their corpus. 
We demonstrate that operators can deploy TopoGuard without labeled attack data.
In extra experiments, we set $\tau$ at the 99th percentile of benign query scores, a standard threshold-calibration procedure for FPR-constrained binary classification that does not require attack labels. 
Across all four detectors and both datasets, this label-free calibration matches supervised recall within $\pm 0.07\%$ (Appendix~\ref{sec:self-supervised-appendix}), well below the bootstrap standard deviation ($\pm 0.4$-$0.5$\%). 
Thus, threshold calibration for a target FPR depends only on the negative class distribution, and attack labels are not strictly required. 
Our contribution is empirical as we verify that detection recall is preserved under this label-free procedure on real attack distributions.

\section{Conclusion}
\label{sec:conclusion}
In this work, we introduced split-knowledge attacks, a practical attack in RAG systems, where the retrieved individual benign documents combine to induce harmful associations.
We empirically verified that the split-knowledge attacks are effective against existing defense mechanisms because they are structurally invisible to per-document filters.
To fill the gap, TopoGuard is designed to detect the retrieved contexts by topological fragmentation, with provable detection guarantees based on the graph's spectral gap and conductance.
Extensive experiments are conducted on multiple baselines and two retrieval datasets.
Specifically, on HotpotQA, our TopoGuard catches approximately 21$\times$ more attacks than LlamaGuard at 1\% FPR, and the hybrid TopoGuard remains the most robust defense under different adaptive adversaries.
As RAG systems are widely deployed, TopoGuard pushes toward more reliable RAG by catching attacks that per-document filters miss at sub-millisecond cost and without labeled attacks.

\section{Limitations}
\label{sec:limitations}
As a frontier defense in this line of research, TopoGuard still exhibits several limitations that constrain its deployment in real-world settings.
First, it targets compositional attacks only.
On single-passage attacks like PoisonedRAG, where the payload sits in one document, our methods will not be a good fit (Appendix~\ref{app:poisonedrag}). 
Thus, we recommend deploying TopoGuard as a complement to per-document filters, not a replacement.
Second, an adversary who places malicious content within the same semantic cluster as the legitimate retrieval (same-community mimicry, Section~\ref{sec:exp}) eliminates the structural anomaly TopoGuard relies on, dropping recall to roughly 1\%.
This is a universal limitation on topological defenses generally.
Future work should pair TopoGuard with retrieval-score consistency checks or query-conditioned detection to close these gaps. 

\bibliographystyle{plain} 
\bibliography{references}

\appendix

\section{Threat Model Validation}
\label{app:threat-validation}
\subsection{Experimental Setup}
We generate 50 adversarial examples using GPT-4o-mini across five domains: science, geography, history, technology, and business. 
Each example contains a multi-hop question, a correct answer, a target false answer, and two document sets (3 documents each, 2-3 sentences per document) generated by GPT-4o mini~\cite{openai2024gpt4omini}.
Legitimate documents form a coherent reasoning chain supporting the correct answer.
Adversarial documents follow a three-part structure: (1) entity introduction without revealing the correct answer, (2) false answer description in isolation, and (3) a semantic bridge creating domain overlap between the two. 
All documents are factually accurate; the adversarial signal exists solely in their composition.
We evaluate using \texttt{Llama-3-8B-Instruct}~\cite{grattafiori2024llama} with deterministic decoding (temperature $0.0$, \texttt{max\_new\_tokens=150}) and a system prompt instructing the model to answer only from the provided documents. 
A response is counted as an attack success if the false answer string appears in the model output.

\textbf{Key Findings.}
Llama-3-8B-Instruct achieves 96\% accuracy on legitimate contexts (48/50). Under attack, ASR reaches 34\% (17/50, 95\% CI: [21\%, 47\%], $p < 0.001$), while the model declines to answer in 36\% of cases. 
Out of the 17 successful attacks, 13 (76.5\%) involve false entity association and 4 (23.5\%) involve spurious reasoning, classified by whether the model output contained causal connectives or exceeded 50 tokens. 
These results confirm that split-knowledge attacks pose a genuine threat to safety-tuned RAG systems, and that text-based moderation alone is insufficient. 
We note that this experiment validates compositional manipulation in isolation and end-to-end validation with a live retrieval corpus is left for future work.

\section{Hyperparameter Validation}
\label{sec:hyperparameters}

\subsection{Hybrid Weight Selection ($\alpha$ Sensitivity)}
\label{app:alpha_sweep}

The hybrid combines spectral and entity signals as $s = \alpha \cdot s_{\lambda_2} + (1-\alpha) \cdot s_{\text{entity}}$.
We sweep $\alpha$ on HotpotQA using the same calibration protocol as Table~\ref{tab:exp1-detections}: $\tau$ is calibrated on full DEV at the $1\%$ FPR target, and metrics are computed on the held-out 10k test set with bootstrap CIs over 1000 resamples. 
Results are in Table~\ref{tab:alpha_sweep}.

On standard attacks, recall peaks near $\alpha=0.25$ (AUROC $0.955$,Recall@$1\%$FPR $39.60\%$). We nonetheless select $\alpha=0.40$ because it dominates $\alpha=0.25$ on every adaptive adversary
(Table~\ref{tab:alpha_adaptive}). 
It also has lower benign FPR on multi-hop queries ($0.49\%$ vs $0.85\%$ overall; $1.94\%$ vs $3.39\%$ on the high-doc-gap Q4 slice; Table~\ref{tab:exp2-fpr}). 
The trade-off is a small loss in raw recall on standard attacks for substantially better adaptive robustness and benign FPR. This matches the production deployment goals stated in Section~\ref{sec:deployment}.

The optimum at $\alpha=0.40$ places $60\%$ of the weight on entity overlap, indicating entity disjointness is the stronger individual signal for split-knowledge attacks. 
Pure entity ($\alpha = 0$) cannot land precisely at $1\%$ FPR because Jaccard scores take only a few discrete values; the spectral term contributes a continuous score that lets the combined detector hit the operating point exactly. 
Performance is stable across a wide range as AUROC stays above $0.94$ for every $\alpha \in [0.10, 0.75]$.

\begin{table}[h]
\centering
\caption{$\alpha$ sensitivity on HotpotQA. $\tau$ is calibrated on full DEV atthe 1\% FPR target. Attack metrics are reported on the held-out 10k test set with bootstrap mean $\pm$ std over 1000 resamples. Benign FPR is from the 2474 benign multi-hop queries in Experiment~\ref{sec:fpr}.
Recall and FPR valuesare in percent. Bold marks the selected $\alpha=0.40$. $\alpha=1$ is pure $\lambda_2$; $\alpha=0$ is pure entity overlap.}
\label{tab:alpha_sweep}
\begin{tabular}{lccccc}
\toprule
$\alpha$ & AUROC & Recall@1\%FPR & Attack FPR & Benign FPR & Benign Q4 FPR \\
\midrule
0.00          & $0.836 \pm 0.003$         & $0.00 \pm 0.00$           & 0.00          & 0.00          & 0.00          \\
0.10          & $0.943 \pm 0.002$         & $39.09 \pm 0.51$          & 0.68          & 0.85          & 3.39          \\
0.25          & $0.955 \pm 0.001$         & $39.60 \pm 0.52$          & 0.68          & 0.85          & 3.39          \\
\textbf{0.40} & $\mathbf{0.952 \pm 0.001}$ & $\mathbf{32.63 \pm 0.49}$ & \textbf{0.44} & \textbf{0.49} & \textbf{1.94} \\
0.50          & $0.949 \pm 0.001$         & $31.09 \pm 0.48$          & 0.40          & 0.49          & 1.94          \\
0.60          & $0.946 \pm 0.002$         & $31.26 \pm 0.48$          & 0.42          & 0.57          & 2.26          \\
0.75          & $0.941 \pm 0.002$         & $30.38 \pm 0.48$          & 0.44          & 0.40          & 1.62          \\
0.90          & $0.936 \pm 0.002$         & $26.44 \pm 0.46$          & 0.42          & 0.40          & 1.62          \\
1.00          & $0.933 \pm 0.002$         & $26.55 \pm 0.46$          & 0.43          & 0.40          & 1.62          \\
\bottomrule
\end{tabular}
\end{table}

\begin{table}[h]
\centering
\caption{Adaptive adversary recall ($\%$) at fixed threshold $\tau$ for
$\alpha=0.40$ vs $\alpha=0.25$. Base coloum is non-adaptive attack from Experiment 4}
\label{tab:alpha_adaptive}
\begin{tabular}{lccccc}
\toprule
$\alpha$ & Base & Random & Best-of-$N$ & Chain & SameCom \\
\midrule
$0.25$         & $22.5$          & $38.0$          & $24.1$          & $35.9$          & $1.1$ \\
$\mathbf{0.40}$ & $\mathbf{33.0}$ & $\mathbf{56.8}$ & $\mathbf{37.5}$ & $\mathbf{53.3}$ & $\mathbf{1.4}$ \\
\midrule
$\Delta$ ($0.40 - 0.25$) & $+10.5$ & $+18.8$ & $+13.4$ & $+17.4$ & $+0.3$ \\
\bottomrule
\end{tabular}
\end{table}

\subsection{\textit{K}-NN Graph Construction (\textit{k} Sensitivity)}
\label{app:k-validation}

We validate the choice of $k$ for $k$-nearest-neighbor graph construction  by sweeping $k \in \{2, 4, 6, 8, 10\}$ and measuring development-set  AUC and score separation (attack mean minus safe mean) for both TopoGuard-$\lambda_2$ and TopoGuard-Conductance. 
Selection is performed  on the dev set only; the test set is held out for final evaluation. Results are shown in Figure~\ref{fig:k_sweep}.

Both AUC metrics improve sharply from $k=2$ to $k=4$ (TopoGuard-$\lambda_2$: 0.891 → 0.932; TopoGuard-Conductance: 0.889 → 0.935), then plateau. 
At $k=2$, the graph is too sparse: many nodes are weakly connected, which obscures the structural signal. 
By $k=4$, the graph captures enough local connectivity for both detectors to reliably separate attacks from safe contexts. 
Beyond $k=4$, AUC changes by less than $0.005$ for both detectors, indicating saturation. 
We therefore select $k=4$ as it achieves the highest dev AUC at minimal 
computational cost. 
All main experiments use this value.

\begin{figure}[h]
    \centering
    \includegraphics[width=\linewidth]{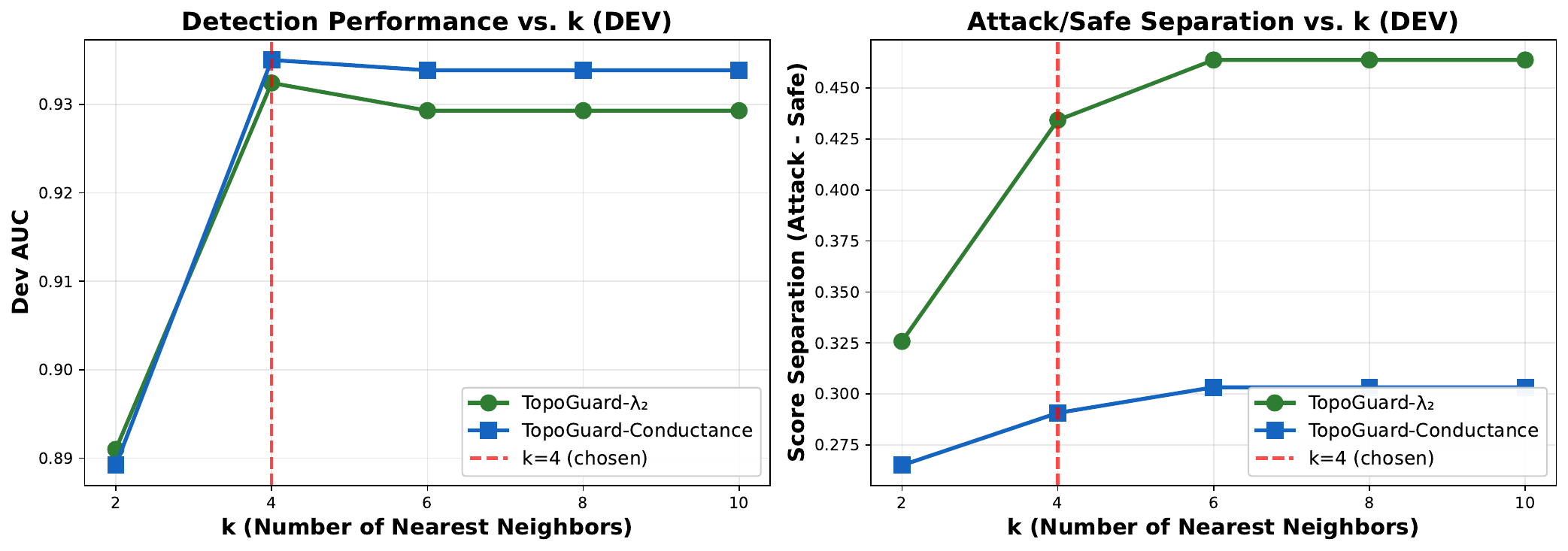}
    \caption{Sensitivity of TopoGuard-$\lambda_2$ and TopoGuard-Conductance 
    to the number of nearest neighbors $k$, evaluated on HotpotQA dev. 
    \textbf{Left:} dev AUC. \textbf{Right:} score separation between attack 
    and safe contexts. Both metrics stabilize at $k=4$ (dashed red line), 
    which is used in all experiments. Selection performed on dev only; 
    test set held out.}
    \label{fig:k_sweep}
\end{figure}

\section{Related Work}
\label{app:related}

\textbf{RAG Security and Corpus Poisoning.}
PoisonedRAG~\citep{zou2025poisonedrag} demonstrates that injecting a  small number of adversarial passages into a retrieval corpus can reliably manipulate LLM outputs.
Subsequent work has explored gradient-based optimization of poisoned documents~\citep{zhong2023poisoning}, black-box corpus attacks~\citep{shafran2024machine}, and single-document attacks via LLM-generated chain-of-evidence reasoning~\citep{chang2025oneshot}. 
Beyond text-only RAG, GragPoison ~\citep{liang2025graphrag} demonstrates that graph-augmented RAG systems are vulnerable to attacks that inject competing relational claims into the underlying knowledge graph. 
Prompt injection attacks~\citep{perez2022ignore, greshake2023not} manipulate the model through the retrieved context itself rather than the corpus.
Our split-knowledge attack differs from all of these: each injected document is individually benign and passes standard content filters. 
The adversarial signal exists only in the composition of retrieved documents, making it invisible to any per-document defense. 
Table~\ref{tab:attack-comparison} summarizes this distinction.

\begin{table}[h]
\centering
\caption{Comparison of representative RAG attack threat models.}
\label{tab:attack-comparison}
\small
\setlength{\tabcolsep}{6pt}
\begin{tabular}{@{}lccc@{}}
\toprule
\textbf{Attack} & \textbf{Malicious passage?} & \textbf{Optimization?} & \textbf{Per-passage filter?} \\
\midrule
PoisonedRAG~\citep{zou2025poisonedrag}              & Yes    & Embedding       & Sometimes \\
Gradient poisoning~\citep{zhong2023poisoning}       & Yes    & Gradient        & Sometimes \\
AuthChain~\citep{chang2025oneshot}                  & Yes    & LLM-generated   & Sometimes \\
GragPoison~\citep{liang2025graphrag}                & Yes    & Relation-based  & Sometimes \\
Prompt injection~\citep{greshake2023not}            & Hidden & No              & Usually \\
\textbf{Split-Knowledge (Ours)}                     & \textbf{No}  & \textbf{Retrieval-only} & \textbf{Never} \\
\bottomrule
\end{tabular}
\end{table}

\textbf{Content Moderation and Safety Filters.}
LlamaGuard~\cite{metallamaguard2} and its successors are instruction-tuned LLMs that classify individual inputs as safe or unsafe.
RoBERTa-based hate speech detectors~\cite{vidgen2021dynabench} and Perspective API operate similarly on single text inputs.
We show empirically that LlamaGuard, RoBERTa-based detectors, and LLM-as-a-Judge approaches all achieve near-random performance 
(AUROC $\approx 0.5$) on split-knowledge attacks: the adversarial signal is structural, not lexical.

\textbf{Graph-Based Anomaly Detection.}
Spectral methods are standard in graph anomaly detection~\cite{akoglu2015graph}, typically on static graphs like social networks or citation networks where anomalies are individual nodes or subgraphs.
The Fiedler value and graph conductance have been used to detect community structure~\cite{newman2006modularity}.
Our setting differs: we build a fresh semantic similarity graph per 
query, and detection uses the entire graph's topology.
To our knowledge, we are the first to apply Cheeger-style certificates to retrieval safety.

\textbf{Retrieval-Augmented Generation.}
RAG grounds LLM outputs in retrieved documents~\citep{lewis2020rag, shuster2021retrieval}.
Multi-hop benchmarks like HotpotQA~\citep{yang2018hotpotqa} and MuSiQue~\citep{trivedi2022musique} require combining multiple documents per query.
Our threat model targets these systems: the need to combine multiple documents is exactly what creates the compositional vulnerability.

\textbf{Spectral Graph Theory and Cheeger Inequalities.}
Our theory builds on Cheeger's inequality~\citep{cheeger1970lower, chung1997spectral}, which relates the spectral gap to graph conductance.
Our proofs use Weyl's inequality~\citep{horn2012matrix} for perturbation and Matrix Bernstein concentration~\citep{tropp2012user}.
Louvain community detection~\citep{blondel2008louvain} underlies our modularity baseline.

\section{Extended Proofs}
\label{app:extended-proofs}

\subsection{Proof of Theorem~\ref{thm:stability} (Spectral Stability)}

Let $G$ be the symmetrized $k$-NN similarity graph from observed embeddings $\{\mathbf{x}_i\}$, and $G^*$ a reference graph from embeddings $\{\mathbf{x}_i^*\}$.

Assume:
\begin{enumerate}
\item Bounded noise: $\|\mathbf{x}_i - \mathbf{x}_i^*\|_2 \leq \sigma$ for all $i$.
\item $k$-NN margin condition: $\sigma$ is small relative to the gap between each point's $k$-th and $(k{+}1)$-th nearest-neighbor distances, so that $G$ and $G^*$ share the same edge set.
\item Bounded symmetrized degree: each node has at most $ck$ neighbors for some constant $c$.
\end{enumerate}

Then the spectral gap satisfies:
$$|\lambda_2(G) - \lambda_2(G^*)| = O\!\left(\frac{k\sigma}{\delta_{\min}^{2}}\right),$$
where $\delta_{\min}$ is the minimum weighted node degree.

\textit{On Assumption 3.} The symmetrized $k$-NN graph can in principle have hub nodes with degree $O(n)$. 
Assumption 3 requires the embedding distribution to have bounded doubling dimension. 
On real sentence embeddings this is empirically benign, with observed maximum degrees within a small constant of $k$.

\begin{proof}
We bound the operator-norm perturbation $\|L - L^*\|_2$ of the normalized Laplacian, then apply Weyl's inequality~\citep{horn2012matrix} to bound the eigenvalue perturbation $|\lambda_2(G) - \lambda_2(G^*)|$. 
The proof proceeds in four steps: 
(i) bound the edge-weight perturbation $\|W - W^*\|$; 
(ii) bound the degree perturbation $\|D - D^*\|$; (iii) bound the inverse-square-root perturbation $\|D^{-1/2} - (D^*)^{-1/2}\|$; 
(iv) combine these into the Laplacian bound $\|L - L^*\|$.

\textbf{Edge-weight perturbation.}
For unit-norm embeddings, by Cauchy-Schwarz and Assumption 1:
\begin{equation}
|\langle \mathbf{x}_i, \mathbf{x}_j \rangle - \langle \mathbf{x}_i^*, \mathbf{x}_j^* \rangle| \leq 2\sigma.
\label{eq:cosine-perturbation}
\end{equation}
By Assumption 2 (margin condition), $G$ and $G^*$ share the same edge set, so $W - W^*$ is nonzero only on edges. 
The clipping 
$w_{ij} = \max(0, \langle \mathbf{x}_i, \mathbf{x}_j \rangle)$ is 1-Lipschitz, so the bound in (\ref{eq:cosine-perturbation}) applies entrywise to $W - W^*$.

By Assumption 3, each row of $W - W^*$ has at most $ck$ nonzero entries, each bounded by $2\sigma$. 
Using the standard row-sum bound on the operator norm of symmetric matrices, 
$\|M\|_2 \leq \|M\|_\infty := \max_i \sum_j |M_{ij}|$:
\begin{equation}
\|W - W^*\|_2 \;\leq\; \|W - W^*\|_\infty \;\leq\; ck \cdot 2\sigma \;=\; 2ck\sigma.
\label{eq:W-perturbation}
\end{equation}

\textbf{Degree perturbation.}
The weighted degree of node $i$ is $D_{ii} = \sum_j W_{ij}$, so by the triangle inequality and (\ref{eq:W-perturbation}):
\begin{equation}
|D_{ii} - D_{ii}^*| = \left| \sum_j (W_{ij} - W_{ij}^*) \right| \leq \sum_j |W_{ij} - W_{ij}^*| \leq 2ck\sigma.
\label{eq:degree-entrywise}
\end{equation}
Since $D - D^*$ is diagonal, $\|D - D^*\|_2 = \max_i |D_{ii} - D_{ii}^*| \leq 2ck\sigma$.

\textbf{Inverse-square-root perturbation.}
Consider $f(x) = 1/\sqrt{x}$ on $[\delta_{\min}, \infty)$. Its derivative is $f'(x) = -\tfrac{1}{2}x^{-3/2}$, so the Lipschitz constant on this interval is:
\begin{equation}
\sup_{x \geq \delta_{\min}} |f'(x)| = \tfrac{1}{2}\delta_{\min}^{-3/2}.
\label{eq:lipschitz}
\end{equation}
Since $D$ and $D^*$ are diagonal, applying $f$ entrywise and taking the operator norm:
\begin{equation}
\|D^{-1/2} - (D^*)^{-1/2}\|_2 
\leq \tfrac{1}{2}\delta_{\min}^{-3/2} \cdot \|D - D^*\|_2 
= O\!\left(\frac{k\sigma}{\delta_{\min}^{3/2}}\right),
\label{eq:Dinv-perturbation}
\end{equation}
combining (\ref{eq:lipschitz}) and (\ref{eq:degree-entrywise}).

\textbf{Laplacian perturbation.}
The normalized Laplacian is $L = I - D^{-1/2} W D^{-1/2}$. Let $A = D^{-1/2}$ and $A^* = (D^*)^{-1/2}$. We expand $L - L^*$ via the algebraic identity:
\begin{equation}
A W A - A^* W^* A^* = (A - A^*) W A + A^*(W - W^*) A + A^* W^* (A - A^*).
\label{eq:three-term}
\end{equation}

From a standard three-term telescoping,
$A W A - A^* W^* A^*$ as 
$(A W A - A^* W A) + (A^* W A - A^* W^* A) + (A^* W^* A - A^* W^* A^*)$, 

We bound each of the three terms. We use:
\begin{itemize}
\item $\|W\|_2 = O(1)$ (bounded weights $w_{ij} \in [0,1]$ with $O(k)$ nonzero entries per row, treating $k$ as a constant).
\item $\|A\|_2 = \|D^{-1/2}\|_2 \leq \delta_{\min}^{-1/2}$ (operator norm of a diagonal matrix is the max diagonal entry).
\item $\|A - A^*\|_2 = O(k\sigma / \delta_{\min}^{3/2})$ from (\ref{eq:Dinv-perturbation}).
\item $\|W - W^*\|_2 = O(k\sigma)$ from (\ref{eq:W-perturbation}).
\end{itemize}

Applying these bounds to the three terms of (\ref{eq:three-term}) using ($\|XY\|_2 \leq \|X\|_2 \|Y\|_2$):
\begin{align*}
\|(A - A^*) W A\|_2 &\leq \|A - A^*\|_2 \|W\|_2 \|A\|_2 = O\!\left(\tfrac{k\sigma}{\delta_{\min}^{3/2}}\right) \cdot O(1) \cdot O\!\left(\tfrac{1}{\delta_{\min}^{1/2}}\right) = O\!\left(\tfrac{k\sigma}{\delta_{\min}^{2}}\right),\\
\|A^*(W - W^*) A\|_2 &\leq \|A^*\|_2 \|W - W^*\|_2 \|A\|_2 = O\!\left(\tfrac{1}{\delta_{\min}^{1/2}}\right) \cdot O(k\sigma) \cdot O\!\left(\tfrac{1}{\delta_{\min}^{1/2}}\right) = O\!\left(\tfrac{k\sigma}{\delta_{\min}}\right),\\
\|A^* W^* (A - A^*)\|_2 &\leq \|A^*\|_2 \|W^*\|_2 \|A - A^*\|_2 = O\!\left(\tfrac{k\sigma}{\delta_{\min}^{2}}\right).
\end{align*}
For small $\delta_{\min}$, the first and third terms (scaling as 
$\delta_{\min}^{-2}$) dominate the second (scaling as 
$\delta_{\min}^{-1}$). By the triangle inequality:
\begin{equation}
\|L - L^*\|_2 \leq O\!\left(\frac{k\sigma}{\delta_{\min}^{2}}\right).
\label{eq:L-perturbation}
\end{equation}

\textbf{Conclusion.} Weyl's inequality~\citep{horn2012matrix} states that for symmetric matrices, $|\lambda_i(L) -\lambda_i(L^*)| \leq \|L - L^*\|_2$ for every eigenvalue index $i$. 
Applied to $i = 2$:

\begin{equation}
|\lambda_2(G) - \lambda_2(G^*)| \leq \|L - L^*\|_2 = O\!\left(\frac{k\sigma}{\delta_{\min}^{2}}\right),
\end{equation}
combining (\ref{eq:L-perturbation}) with Weyl's inequality.
\end{proof}

\subsection{Proof of Theorem~\ref{thm:detection} (Detection Certificate)}

Suppose reference attack graphs have conductance $\phi(G^*) \leq \delta$ and reference legitimate graphs have $\phi(G^*) \geq \epsilon$, with 
$\epsilon \geq 2\sqrt{\delta}$. 
Assume embedding noise components are sub-Gaussian with parameter $\sigma$ and the $k$-NN margin condition of Theorem~\ref{thm:stability}. 
Let $\tau_\lambda = \delta + \epsilon^2/4$ denote the threshold in $\lambda_2$-space, with corresponding score-space threshold $\tau = 1 - \min(\tau_\lambda, 1)$, and margin $\Delta = (\epsilon^2 - 4\delta)/4 > 0$.

Then:
\begin{align*}
\Pr_{D \sim \text{Attack}}[s(D) > \tau]
  &\geq 1 - \exp\!\left(-\Omega\!\left(\tfrac{n\Delta^2}{\sigma^2}\right)\right), \\
\Pr_{D \sim \text{Legit}}[s(D) > \tau]
  &\leq \exp\!\left(-\Omega\!\left(\tfrac{n\Delta^2}{\sigma^2}\right)\right).
\end{align*}
The first bound is the true positive rate (TPR) at threshold $\tau$. 
The second is the false positive rate (FPR).
Both improve exponentially with the number of retrieved units $n$.

\textit{On the concentration bound.}
The bound $\Pr(\|L - L^*\|_2 \geq t) \leq \exp(-\Omega(nt^2/\sigma^2))$ used in the proof treats the Laplacian perturbation as if it concentrated like a sum of independent random matrices. 
In a $k$-NN similarity graph, edges are coupled through degree normalization and neighbor selection, so the formal preconditions of Matrix Bernstein-type bounds do not apply directly.
We treat this as a heuristic concentration that captures the qualitative scaling. 
A fully rigorous concentration analysis under $k$-NN graph structure is left to future work.

\begin{proof}
We first translate the conductance bounds ($\phi \leq \delta$ for attacks, $\phi \geq \epsilon$ for legit) into bounds on $\lambda_2$ using Cheeger's inequality. 
The condition $\epsilon \geq 2\sqrt{\delta}$ guarantees a positive margin between attack and legitimate $\lambda_2$ regimes. 
Theorem~\ref{thm:stability} plus a heuristic concentration bound then gives exponential tail bounds on the probability that the noisy $\lambda_2(G)$ deviates from the ideal value by more than the margin. 
Finally, we convert these $\lambda_2$-space bounds to score-space bounds (FPR and TPR) using the monotonicity of the scoring function.

\textbf{Cheeger bounds on $\lambda_2$.}
Cheeger's inequality~\citep{chung1997spectral} relates the spectral gap to graph conductance via the two-sided bound $\phi^2/2 \leq \lambda_2 \leq 2\phi$. 
Applying both directions:
\begin{align}
\text{Attack regime:} \quad &\lambda_2(G^*) \leq 2\phi(G^*) \leq 2\delta, \label{eq:attack-lambda}\\
\text{Legitimate regime:} \quad &\lambda_2(G^*) \geq \phi(G^*)^2/2 \geq \epsilon^2/2. \label{eq:legit-lambda}
\end{align}

\textbf{Threshold and margin.}
The condition $\epsilon \geq 2\sqrt{\delta}$ rearranges to $\epsilon^2 \geq 4\delta$, which guarantees that the attack ceiling $2\delta$ lies strictly below the legitimate floor $\epsilon^2/2$:
\begin{equation}
2\delta < \delta + \frac{\epsilon^2}{4} < \frac{\epsilon^2}{2}.
\label{eq:tau-position}
\end{equation}
The threshold $\tau_\lambda = \delta + \epsilon^2/4$ is the midpoint of the gap $[2\delta, \epsilon^2/2]$, with equal margin to either side:
\begin{equation}
\Delta := \tau_\lambda - 2\delta = \frac{\epsilon^2}{2} - \tau_\lambda = \frac{\epsilon^2 - 4\delta}{4} > 0.
\label{eq:margin}
\end{equation}

\textbf{Detection failure events.}
By Theorem~\ref{thm:stability}, the observed $\lambda_2(G)$ deviates from the ideal $\lambda_2(G^*)$ by $|\lambda_2(G) - \lambda_2(G^*)| \leq \|L - L^*\|_2$. Combining with the regime bounds (\ref{eq:attack-lambda})--(\ref{eq:legit-lambda}):
\begin{itemize}
\item For an attack to be missed (false negative), we need 
$\lambda_2(G) > \tau_\lambda$. 
Since $\lambda_2(G^*) \leq 2\delta$ for attacks and $\tau_\lambda - 2\delta = \Delta$, this requires $\|L - L^*\|_2 > \Delta$.
\item For a legitimate query to be flagged (false positive), we need $\lambda_2(G) < \tau_\lambda$. Since $\lambda_2(G^*) \geq \epsilon^2/2$ for legit and $\epsilon^2/2 - \tau_\lambda = \Delta$, this also requires $\|L - L^*\|_2 > \Delta$.
\end{itemize}
Both failure events thus reduce to the same Laplacian-perturbation event:\begin{equation}
\{\|L - L^*\|_2 > \Delta\}.
\label{eq:failure-event}
\end{equation}

\textbf{Concentration.}
Standard concentration arguments on the operator-norm perturbation of the normalized Laplacian under sub-Gaussian embedding noise yield:
\begin{equation}
\Pr(\|L - L^*\|_2 \geq t) \leq \exp\!\left(-\Omega\!\left(\frac{nt^2}{\sigma^2}\right)\right).
\label{eq:concentration}
\end{equation}
Setting $t = \Delta$ and using the failure-event reduction (\ref{eq:failure-event}):
\begin{align}
\Pr_{D \sim \text{Attack}}[\lambda_2(G) > \tau_\lambda] 
  &\leq \exp\!\left(-\Omega\!\left(\frac{n\Delta^2}{\sigma^2}\right)\right), 
  \label{eq:lambda-FN}\\
\Pr_{D \sim \text{Legit}}[\lambda_2(G) < \tau_\lambda] 
  &\leq \exp\!\left(-\Omega\!\left(\frac{n\Delta^2}{\sigma^2}\right)\right).
  \label{eq:lambda-FP}
\end{align}

\textbf{Score-space conversion (FPR and TPR).}
The detector score is $s(D) = 1 - \min(\lambda_2(G), 1)$, which is monotone decreasing in $\lambda_2$. 
Equivalently, $\lambda_2(G) > \tau_\lambda \iff s(D) < 1 - \tau_\lambda = \tau$ (when $\tau_\lambda \leq 1$). 
Substituting into (\ref{eq:lambda-FN})-(\ref{eq:lambda-FP}):

\textit{True positive rate} ($\Pr[s(D) > \tau \mid \text{attack}]$):
\begin{equation}
\Pr_{D \sim \text{Attack}}[s(D) > \tau] 
  = 1 - \Pr_{D \sim \text{Attack}}[s(D) \leq \tau] 
  \geq 1 - \exp\!\left(-\Omega\!\left(\frac{n\Delta^2}{\sigma^2}\right)\right).
\end{equation}

\textit{False positive rate} ($\Pr[s(D) > \tau \mid \text{legit}]$):
\begin{equation}
\Pr_{D \sim \text{Legit}}[s(D) > \tau] 
  = \Pr_{D \sim \text{Legit}}[\lambda_2(G) < \tau_\lambda] 
  \leq \exp\!\left(-\Omega\!\left(\frac{n\Delta^2}{\sigma^2}\right)\right).
\end{equation}
These are the stated bounds: TPR approaches 1 and FPR approaches 0 exponentially in $n$.
\end{proof}

\subsection{Proof of Corollary~\ref{cor:sample_complexity} (Sample Complexity)}

\begin{proof}
By Theorem~\ref{thm:detection}, both the false negative and false positive probabilities are bounded by 
$\exp(-\Omega(n\Delta^2/\sigma^2))$ where 
$\Delta = (\epsilon^2 - 4\delta)/4$. Requiring this bound to be at most $\min(\alpha, \beta)$:
\begin{equation}
\exp\!\left(-\Omega\!\left(\frac{n\Delta^2}{\sigma^2}\right)\right) \leq \min(\alpha, \beta).
\label{eq:sample-bound}
\end{equation}
Taking logarithms of both sides and rearranging:
\begin{equation}
n \geq \Omega\!\left(\frac{\sigma^2}{\Delta^2} \log \frac{1}{\min(\alpha, \beta)}\right).
\label{eq:sample-solve}
\end{equation}
Substituting $\Delta^2 = (\epsilon^2 - 4\delta)^2/16$ and absorbing the constant $16$ into the big-$O$:
\begin{equation}
n = O\!\left(\frac{\sigma^2}{(\epsilon^2 - 4\delta)^2} \log \frac{1}{\min(\alpha, \beta)}\right).
\end{equation}
\end{proof}

\section{Theory Validation}
\label{app:theory_validation}

We empirically test the two assumptions behind our detection guarantees: the bounded-noise model (Theorem~\ref{thm:stability}) and the sample complexity bound (Corollary~\ref{cor:sample_complexity}).

\textbf{Noise model.}
Theorem~\ref{thm:stability} assumes bounded encoder noise; the concentration argument in Theorem~\ref{thm:detection} additionally treats this noise as sub-Gaussian. 
We do not observe the ideal semantic geometry directly, so we report two proxies. 
(i)~The encoder (\texttt{all-mpnet-base-v2}) is deterministic at 
inference: re-encoding the same text yields identical outputs, so the encoder contributes no stochastic component. 
(ii)~Across 200 HotpotQA safe contexts, the pairwise cosine-similarity 
distribution (Figure~\ref{fig:theory_val}a) is approximately Gaussian ($\mu = 0.34$, $\sigma = 0.18$), indicating that embedding geometry is well-behaved and not dominated by heavy-tailed outliers. 
This supports the sub-Gaussian assumption at the distributional level. 
The effective $\sigma$ (the gap between observed and ideal embeddings) is not directly measurable, and our bounds hold under this assumption rather than as direct empirical guarantees.

\textbf{Sample complexity.}
Corollary~\ref{cor:sample_complexity} predicts that recall improves with context size. 
We vary sentences-per-document $s \in \{1,\dots,6\}$, giving total context sizes $2s$, and recalibrate $\tau$ on the dev set at each $s$(Figure~\ref{fig:theory_val}b). 
Recall rises from $21.1\%$ at $s{=}1$ to $26.6\%$ at $s{=}3$, then plateaus. The plateau reflects HotpotQA's fixed 6-sentence contexts rather than a fundamental method limit.

\textbf{Boundary case at $s{=}2$.}
At $s{=}2$ (4 nodes), $k_{\text{eff}}=\min(k, n{-}1)=3$ forces a complete graph. 
Recall spikes to $75.3\%$, but FPR rises to $1.15\%$, slightly above the $1\%$ target. This is consistent with theory that the detection improves as the graph approaches completeness, but threshold calibration becomes geometry-sensitive at small $n$. 
Our main experiments use $s{=}3$ where this regime is avoided.

\begin{figure}[h]
    \centering
    \includegraphics[width=\linewidth]{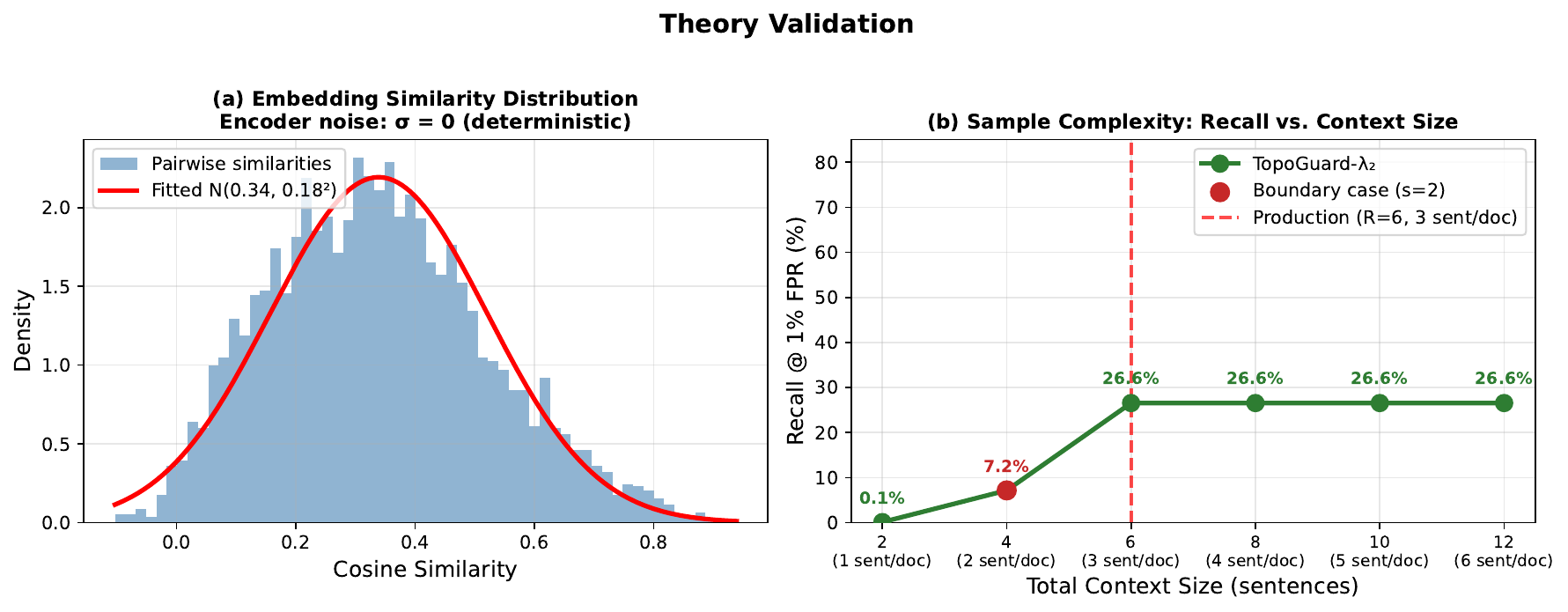}
    \caption{Theory validation on HotpotQA.
    \textbf{(a)} Pairwise cosine-similarity distribution across safe
    contexts is approximately Gaussian, consistent with sub-Gaussian
    embedding geometry.
    \textbf{(b)} Recall vs.\ context size $s$ (sentences per document).}
    \label{fig:theory_val}
\end{figure}

\section{Evaluation on PoisonedRAG Attacks}
\label{app:poisonedrag}

We test TopoGuard on single-passage attacks from PoisonedRAG~\citep{zou2025poisonedrag}
to confirm that our scope is compositional attacks, not content-level injection.
We use the pre-computed adversarial contexts released with PoisonedRAG (HotpotQA,Contriever retriever, $N{=}100$ attacks) and reuse Experiment 1's HotpotQA safe set.

\textbf{Setup.}All contexts are fixed at 6 sentences to match our main calibration.
Attack-1 replaces one clean sentence with one adversarial sentence (1:5 ratio; PoisonedRAG's realistic threat model). 
Attack-3 replaces three clean sentences with three adversarial sentences (3:3 ratio; matches our split-knowledge composition). 
For the safe set, we use 2499 HotpotQA test contexts (similar in construction to Section~\ref{sec:fpr}'s 2474 question safe pool, with no material effect on FPR estimates).
Thresholds are frozen from Experiment 1 with no re-tuning. 
The frozen $\tau$ yields FPRs of 0.60\% ($\lambda_2$), 0.68\% (Cond), 0.80\% (Mod), confirming calibration transfers to this pipeline.

\begin{table}[h]
\centering
\small
\caption{TopoGuard on PoisonedRAG. Under faithful 1:5 injection, all detectors score below chance with negative separation. Recall values are in percent.}
\label{tab:poisonedrag}
\begin{tabular}{lccc}
\toprule
\textbf{Method} & \textbf{AUROC} & \textbf{Recall@1\%FPR} & \textbf{Score Sep.} \\
\midrule
\multicolumn{4}{l}{\textit{Attack-1: 5 clean + 1 adv sentence}} \\
$\lambda_2$          & $0.455 \pm 0.041$ & $1.0$ & $-0.016$ \\
Conductance          & $0.333 \pm 0.038$ & $1.0$ & $-0.066$ \\
Modularity           & $0.424 \pm 0.039$ & $0.0$ & $-0.016$ \\
\midrule
\multicolumn{4}{l}{\textit{Attack-3: 3 clean + 3 adv sentences}} \\
$\lambda_2$          & $0.946 \pm 0.015$ & $30.0$ & $+0.463$ \\
Conductance          & $0.950 \pm 0.014$ & $34.0$ & $+0.315$ \\
Modularity           & $0.954 \pm 0.013$ & $40.0$ & $+0.241$ \\
\midrule
\multicolumn{4}{l}{\textit{Reference: split-knowledge on HotpotQA (Exp.~1)}} \\
$\lambda_2$          & $0.932$ & $26.6$ & $+0.342$ \\
Conductance          & $0.935$ & $29.6$ & --- \\
Modularity           & $0.940$ & $35.1$ & --- \\
\bottomrule
\end{tabular}
\end{table}

Attack-1 confirms our scope claim empirically. 
PoisonedRAG optimizes each adversarial passage to embed near the query, which also places it inside the cluster of legitimate retrievals, reducing rather than increasing spectral fragmentation. Hence the negative score separation: TopoGuard's decision boundary is oriented the wrong way for this threat model.
Attack-3 uses the same adversarial text as Attack-1 but at higher injection density.
Performance recovers to AUROC $\approx 0.95$, matching split-knowledge. 
The signal TopoGuard detects is therefore compositional disruption, not PoisonedRAG content itself.
Single-passage and compositional attacks occupy orthogonal regions of the threat space, and defenses are complementary.
TopoGuard should be paired with a per-passage check (e.g., retrieval-score consistency) for full coverage.

\section{MuSiQue Detection Analysis}
\label{app:musique}

\subsection{Why MuSiQue Detection is Harder}

TopoGuard achieves significantly lower recall on MuSiQue (5--7\%) than on HotpotQA (26--35\%).
We show this is a fundamental property of the dataset, not a failure of the method, by analyzing the document gap (the cosine distance) between the mean embeddings of the two retrieved documents per query.
Table~\ref{tab:docgap} shows the mean document gap for safe and attack contexts on both datasets.
On HotpotQA, safe and attack contexts are clearly separated (0.539 vs.\ 0.881, margin of 0.342), giving the spectral detector a strong signal.
On MuSiQue, this margin shrinks to 0.200 (0.717 vs.\ 0.917) because legitimate multi-hop queries naturally span semantically distant domains.

\begin{table}[h]
\centering

\begin{tabular}{lcc}
\toprule
\textbf{Split} & \textbf{Mean Gap} & \textbf{Std} \\
\midrule
HotpotQA safe    & 0.539 & 0.166 \\
HotpotQA attack  & 0.881 & 0.117 \\
\midrule
MuSiQue safe     & 0.717 & 0.200 \\
MuSiQue attack   & 0.917 & 0.104 \\
MuSiQue Q4 safe  & 0.947 & ---   \\
\bottomrule
\end{tabular}
\caption{Mean document gap (1 $-$ cosine similarity) by dataset and split.}
\label{tab:docgap}
\end{table}

\begin{table}[ht]
\centering
\caption{Full FPR (\%) stratification on benign multi-hop queries by semantic distance (doc-gap).}
\label{tab:exp2-fpr-full}
\small
\begin{tabular}{@{}lccccc@{}}
\toprule
\textbf{Method} & \textbf{Overall} & \textbf{Q1 (Low)} & \textbf{Q2} & \textbf{Q3} & \textbf{Q4 (High)} \\
\midrule
\multicolumn{6}{@{}l}{\textit{HotpotQA ($n = 2{,}474$)}} \\
\multicolumn{6}{@{}l}{\textit{Graph-Theoretic Methods (Ours)}} \\
TopoGuard-Cond$(k{=}4)$                      & 0.32 & 0.00 & 0.00 & 0.00 & 1.29 \\
TopoGuard-$\lambda_2(k{=}4)$                 & 0.40 & 0.00 & 0.00 & 0.00 & 1.62 \\
TopoGuard-$\lambda_2$+Entity$(\alpha{=}0.4)$ & 0.49 & 0.00 & 0.00 & 0.00 & 1.94 \\
TopoGuard-Modularity$(k{=}4)$                & 0.81 & 0.00 & 0.00 & 0.00 & 3.23 \\
\midrule
\multicolumn{6}{@{}l}{\textit{Baseline Methods}} \\
GraphAvgWeight$(k{=}4)$ & 1.29 & 0.00 & 0.00 & 0.32 & 4.85 \\
NaiveDensity            & 1.41 & 0.00 & 0.00 & 0.32 & 5.33 \\
\midrule
\multicolumn{6}{@{}l}{\textit{Text-Only Filters}} \\
TextFilter (RoBERTa)        & 0.93 & 0.81 & 1.13 & 0.49 & 1.29 \\
LlamaGuard-2-8B             & 1.41 & 2.10 & 1.13 & 1.13 & 1.29 \\
LlamaGuard-3-8B             & 1.13 & 0.65 & 0.81 & 1.13 & 1.94 \\
LLM-as-a-Judge (Llama-3-8B) & 0.00 & 0.00 & 0.00 & 0.00 & 0.00 \\
\midrule
\multicolumn{6}{@{}l}{\textit{MuSiQue ($n = 1{,}180$)}} \\
\multicolumn{6}{@{}l}{\textit{Graph-Theoretic Methods (Ours)}} \\
TopoGuard-Cond$(k{=}4)$                      & 1.02 & 0.00 & 0.34 & 0.34 & 3.39 \\
TopoGuard-$\lambda_2(k{=}4)$                 & 1.02 & 0.00 & 0.34 & 0.34 & 3.39 \\
TopoGuard-$\lambda_2$+Entity$(\alpha{=}0.4)$ & 0.85 & 0.00 & 0.00 & 0.00 & 3.39 \\
TopoGuard-Modularity$(k{=}4)$                & 1.36 & 0.00 & 0.00 & 0.00 & 5.42 \\
\midrule
\multicolumn{6}{@{}l}{\textit{Baseline Methods}} \\
GraphAvgWeight$(k{=}4)$ & 0.25 & 0.00 & 0.00 & 0.68 & 0.34 \\
NaiveDensity            & 0.51 & 0.00 & 0.00 & 0.34 & 1.69 \\
\midrule
\multicolumn{6}{@{}l}{\textit{Text-Only Filters}} \\
TextFilter (RoBERTa)        & 3.64 & 2.71 & 3.39 & 5.08 & 3.39 \\
LlamaGuard-2-8B             & 0.93 & 1.69 & 0.68 & 0.34 & 1.02 \\
LlamaGuard-3-8B             & 0.34 & 0.00 & 0.00 & 0.34 & 1.02 \\
LLM-as-a-Judge (Llama-3-8B) & 0.00 & 0.00 & 0.00 & 0.00 & 0.00 \\
\bottomrule
\end{tabular}
\end{table}

\subsection{MuSiQue Q4 Queries Resemble Attacks}

Figure~\ref{fig:docgap} shows the full document gap distributions.
The key finding is in Panel~(b): MuSiQue Q4 safe contexts (the 25\% of legitimate queries with the highest document gap) reach a mean gap of 0.947. 
This is nearly identical to HotpotQA attacks (0.881).
Specifically, 76.1\% of MuSiQue Q4 safe gaps exceed the median HotpotQA attack gap, and the two distributions are statistically distinct (KS $p < 0.001$) only because MuSiQue Q4 gaps are actually higher than HotpotQA attacks, not lower.
This means the detector correctly flags MuSiQue Q4 queries as anomalous but they are legitimate.
The low MuSiQue recall is therefore a consequence of operating near the theoretical detection limit: when legitimate queries are as semantically disconnected as attacks, no structural detector can reliably distinguish them without additional context such as the query itself or entity overlap.

These results suggest that query-conditioned detection i.e. using the query to anchor expected document connectivity could resolve the ambiguity for
MuSiQue-style cross-domain reasoning. We leave this extension to future work.

\begin{figure}[h]
    \centering
    \includegraphics[width=\linewidth]{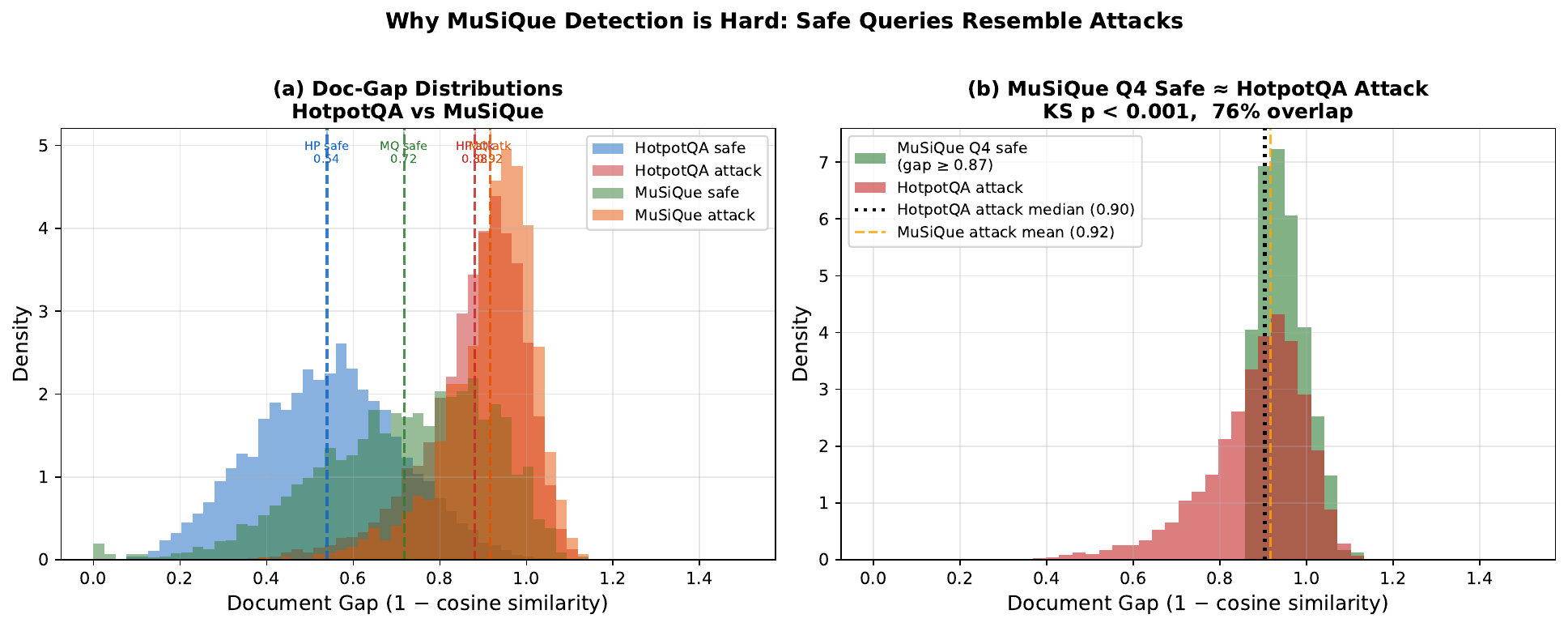}
    \caption{Document gap distributions.
    \textbf{(a)} HotpotQA safe contexts (blue) are well-separated
    from attacks (red) and MuSiQue safe contexts (green) overlap substantially with HotpotQA attacks.
    \textbf{(b)} MuSiQue Q4 safe contexts have gaps nearly identical
    to HotpotQA attacks.}
    \label{fig:docgap}
\end{figure}

\section{Full FPR Stratification by Document Gap}
\label{app:fpr-full}

Table~\ref{tab:exp2-fpr-full} presents the full FPR breakdown by document-gap quartile for both datasets.
Q1 contains the queries with the lowest semantic distance between supporting documents; Q4 contains the highest. 
On HotpotQA, all topological methods achieve exactly 0\% FPR on Q1--Q3, confirming that natural multi-hop reasoning chains are not flagged.
Errors concentrate entirely in Q4, where legitimate queries span the most distant semantic domains. 
MuSiQue shows a similar but noisier pattern due to its deliberate cross-domain design (Appendix~\ref{app:musique}).

\begin{table}[h]
\centering
\caption{TopoGuard across three encoder families on HotpotQA ($k{=}4$, 1\% FPR target). $\tau$ recalibrated per encoder. Bootstrap mean $\pm$ std over 1{,}000 resamples.}
\label{tab:encoder}
\small
\begin{tabular}{@{}llcc@{}}
\toprule
Encoder & Detector & AUROC (\%) $\uparrow$ & Recall@1\%FPR (\%) $\uparrow$ \\
\midrule
MPNet (baseline) & TopoGuard-$\lambda_2$        & 93.2 $\pm$ 0.2 & 26.6 $\pm$ 0.5 \\
MPNet (baseline) & TopoGuard-$\lambda_2$+Entity & 95.2 $\pm$ 0.1 & 32.6 $\pm$ 0.5 \\
\midrule
BGE-small        & TopoGuard-$\lambda_2$        & 91.5 $\pm$ 0.2 & 26.3 $\pm$ 0.4 \\
BGE-small        & TopoGuard-$\lambda_2$+Entity & 94.2 $\pm$ 0.1 & 34.8 $\pm$ 0.5 \\
\midrule
MiniLM-L12-v2    & TopoGuard-$\lambda_2$        & 92.4 $\pm$ 0.2 & 21.3 $\pm$ 0.4 \\
MiniLM-L12-v2    & TopoGuard-$\lambda_2$+Entity & 94.7 $\pm$ 0.1 & 32.7 $\pm$ 0.5 \\
\bottomrule
\end{tabular}
\end{table}

\begin{table}[h]
\centering
\caption{Self-supervised vs supervised calibration. Self-supervised 
uses only the 99th percentile of benign dev scores; no attack labels 
are used. Recall is reported as bootstrap mean $\pm$ std at the 
calibrated $\tau$.}
\label{tab:self-supervised}
\small
\begin{tabular}{@{}llrrr@{}}
\toprule
Dataset & Detector & Sup.\ Recall & SS Recall & $\Delta$ Recall \\
\midrule
HotpotQA & TopoGuard-$\lambda_2$        & 26.55 $\pm$ 0.45 & 26.53 $\pm$ 0.45 & +0.02 \\
HotpotQA & TopoGuard-Cond               & 29.60 $\pm$ 0.47 & 29.54 $\pm$ 0.47 & +0.06 \\
HotpotQA & TopoGuard-Modularity         & 35.09 $\pm$ 0.49 & 35.03 $\pm$ 0.49 & +0.06 \\
HotpotQA & TopoGuard-$\lambda_2$+Entity & 32.63 $\pm$ 0.48 & 32.62 $\pm$ 0.49 & +0.01 \\
\midrule
MuSiQue  & TopoGuard-$\lambda_2$        &  5.70 $\pm$ 0.44 &  5.70 $\pm$ 0.44 & +0.00 \\
MuSiQue  & TopoGuard-Cond               &  5.50 $\pm$ 0.43 &  5.43 $\pm$ 0.43 & +0.07 \\
MuSiQue  & TopoGuard-Modularity         &  7.23 $\pm$ 0.49 &  7.20 $\pm$ 0.49 & +0.03 \\
MuSiQue  & TopoGuard-$\lambda_2$+Entity &  5.53 $\pm$ 0.43 &  5.53 $\pm$ 0.43 & +0.00 \\
\bottomrule
\end{tabular}
\end{table}

\section{Encoder Generalization}
\label{app:encoder-ablation}

We test whether TopoGuard generalizes beyond the MPNet encoder used in our main experiments.
We re-run TopoGuard-$\lambda_2$ and TopoGuard-$\lambda_2$+Entity on HotpotQA with two additional sentence encoders: BGE-small-en-v1.5 (384-dim, retrieval-tuned) and all-MiniLM-L12-v2 (384-dim, distilled).
$\tau$ is recalibrated on the dev set for each encoder using the self-supervised procedure (Section~\ref{sec:self-supervised}).

Table~\ref{tab:encoder} reports the results.
AUROC stays in a tight range across all three encoders (91.5–95.2\%), confirming the topological signal is not specific to MPNet.
The hybrid detector is consistently strong across encoders (94.2–95.2\% AUROC, 32.6–34.8\% Recall@1\%FPR) and is the most encoder-robust configuration.
TopoGuard-$\lambda_2$ alone is more sensitive: BGE-small matches MPNet, while MiniLM drops to 21.3\% recall.
This reflects the known sensitivity of strict-FPR operating points to score-distribution shifts.
For production, we recommend the hybrid detector with MPNet or BGE-small as the encoder.

\section{Label-Free Calibration: Empirical Validation}
\label{sec:self-supervised-appendix}

Threshold calibration for a target false positive rate depends only on the negative-class score distribution; attack labels are not mathematically required. 
We verify empirically that recall is preserved under this standard label-free procedure compared to supervised calibration that uses dev attack labels to confirm the FPR target.

Table~\ref{tab:self-supervised} reports the comparison across all four detectors and both datasets. 
Self-supervised calibration sets $\tau$ at the 99th percentile of benign dev scores. 
The $\Delta$ Recall column shows the gap to supervised calibration; in all 8 settings the gap is below 0.07pp, well within bootstrap noise.

\section{Compute Resources}
\label{sec:compute}
All experiments use a single NVIDIA L40S GPU.
TopoGuard's spectral computations themselves are CPU-runnable in 
sub-millisecond time per query (Figure~\ref{fig:latency}); the GPU is used only for sentence-embedding (all-mpnet-base-v2 in main experiments, BGE-small/MiniLM in the encoder ablation) and for 
LLM-based baselines.

\section{Broader Impact Statement}
\label{app:impact}
TopoGuard is a defensive contribution that makes production RAG systems more resistant to compositional attacks.
It provides a sub-millisecond detection layer that runs alongside existing content moderation, calibrates without labeled attacks, and re-tunes predictably when retrieval configurations change.
By releasing the defense alongside the threat model, we aim to harden RAG safety before split-knowledge attacks are observed in deployed systems.
TopoGuard is not a standalone safety mechanism.
It should be deployed in cascade with per-document content moderation systems.

\newpage
\end{document}